\documentclass{IEEEtran}

\ifCLASSINFOpdf

\else

\fi

\usepackage{booktabs,multirow,makecell}
\usepackage{hyperref}
\usepackage{algorithm,algorithmic}
\usepackage{url}
\usepackage{graphicx}
 \usepackage{float}         % for [H]
 \usepackage{multirow}
 \usepackage[labelfont=bf]{caption}
 \usepackage{subcaption} 
 \usepackage{booktabs}
\usepackage{graphicx}
\usepackage{amsmath}
\usepackage{xcolor}
\usepackage{siunitx}

\usepackage{booktabs}
\usepackage{amsmath,amssymb}
\usepackage{enumitem}
\usepackage{tcolorbox}
\usepackage{listings}
\usepackage{xcolor}

\lstdefinestyle{pytiny}{
  language=Python,
  basicstyle=\ttfamily\small,
  keywordstyle=\color{blue!60!black}\bfseries,
  commentstyle=\color{gray!70},
  stringstyle=\color{green!40!black},
  showstringspaces=false,
  frame=single,
  framerule=0.3pt,
  rulecolor=\color{black!20},
  breaklines=true,
  tabsize=2
}

\usepackage{amsmath,amssymb}
\usepackage{tikz}
\usetikzlibrary{arrows.meta,positioning,fit,calc}

\tikzset{
  >=Latex,
  blk/.style={
    rounded corners=3pt,
    draw=black,
    very thick,
    minimum width=40mm,
    minimum height=9mm,
    align=center
  },
  note/.style={font=\small}
}
\definecolor{cin}{HTML}{F2B8C6}
\definecolor{cfn}{HTML}{F8E1A1}
\definecolor{csu}{HTML}{CBD9FF}
\definecolor{cgvt}{HTML}{E8D7FF}
\definecolor{cconv}{HTML}{BFE3F0}
\definecolor{ccomb}{HTML}{DDEFC5}

\usepackage[T1]{fontenc}
\usepackage{lmodern}         % nicer default font
\usepackage{microtype}       % better kerning / justification
\usepackage{enumitem}        % compact \begin{enumerate}[label=...]\usepackage{subcaption}
\usepackage{amsmath,amsthm}
\usepackage{bm}
\usepackage{amsthm}
\theoremstyle{plain}
\newtheorem{proposition}{Proposition}

\theoremstyle{definition}
\newtheorem{definition}{Definition}

\usepackage{hyperref}
\hypersetup{
  colorlinks=true,
  linkcolor=blue,
  citecolor=blue,
  urlcolor=blue
}

\newcommand{\diag}{\operatorname{diag}}   % diagonal matrix
\usepackage{amsmath,amssymb,amsthm}
\usepackage{mathtools}\usepackage{bm}\usepackage{geometry}
\newtheorem{theorem}{Theorem}

\usepackage[utf8]{inputenc} % allow utf-8 input
\usepackage[T1]{fontenc}    % use 8-bit T1 fonts
\usepackage{microtype}
\usepackage{graphicx}
\usepackage{tabularx}
\usepackage{booktabs} % for professional tables
\usepackage{bm} % For bold math symbols

\usepackage{amsmath}    % For mathematical environments and commands
\usepackage{amssymb}    % For additional symbols like \mathbb and \nabla
\usepackage{amsfonts}   % For font support like \mathbb and \mathcal
\usepackage{mathtools} 
\usepackage{amssymb,amsmath,amsthm}
\usepackage{booktabs}
\usepackage{adjustbox}
\usepackage{siunitx}
\usepackage{array}
\usepackage{booktabs}                  % For nicer rules in tables
\usepackage{caption}                   % For better caption control
\usepackage{subcaption}   
\usepackage{hyperref}
\usepackage{balance}

\usepackage{amsmath}
\usepackage{amssymb}
\usepackage{mathtools}
\usepackage{amsthm}
\usepackage{graphicx}
\usepackage{booktabs}
\usepackage{array}
\usepackage{amsmath}
\usepackage{float}
\usepackage[utf8]{inputenc}
\usepackage{graphicx}
\usepackage{tikz}
\usepackage{natbib} 
\usepackage[capitalize,noabbrev]{cleveref}
\usepackage{wrapfig}
\theoremstyle{plain}

\theoremstyle{definition}

\theoremstyle{remark}
\newtheorem{remark}[theorem]{Remark}

\begin{document}
\title{Graph Variate Neural Networks}
\author{Om~Roy,
        Yashar~Moshfeghi,
        and~Keith~Malcolm~Smith% <-this % stops a space
\thanks{O. Roy, Y. Moshfeghi, and K. M. Smith are with the Department
of Computer and Information Sciences, University of Strathclyde, Glasgow,
G1 1XH, UK, and also affiliated with DeepBrain 
(e-mail: om.roy@deep-brain.tech;.Om Roy is supported by the Engineering and Physical Sciences Research Council (EPSRC) Student Excellence Award (SEA) Studentship provided by the United Kingdom Research and Innovation (UKRI) council, RTSG Grant Number: 12212S220211-116.}% <-this % stops a space
}

\maketitle

\begin{abstract}
  Modeling dynamically evolving spatio-temporal signals is a prominent challenge in the Graph Neural Network (GNN) literature. Notably, GNNs assume an existing underlying graph structure. While this underlying structure may not always exist or is derived independently from the signal, a temporally evolving \textit{functional} network can always be constructed from multi-channel data. Graph Variate Signal Analysis (GVSA) defines a unified framework consisting of a network tensor of instantaneous connectivity profiles against a stable support usually constructed from the signal itself. Building on Graph-Variate Signal Analysis (GVSA) and tools from graph signal processing, we introduce \textbf{Graph-Variate Neural Networks (GVNNs)}: layers that convolve spatio-temporal signals with a signal-dependent connectivity tensor combining a stable long-term support with instantaneous, data-driven interactions. This design captures dynamic statistical interdependencies at each time step without ad-hoc sliding windows and admits an efficient implementation with linear complexity in sequence length. Across forecasting benchmarks, GVNNs consistently outperform strong graph-based baselines and are competitive with widely used sequence models such as LSTMs and Transformers. On EEG motor-imagery classification, GVNNs achieve strong accuracy highlighting their potential for brain–computer interface applications.
\end{abstract}

\begin{IEEEkeywords}
Graph neural networks, graph signal processing, spatio-temporal modeling, graph variate signal analysis
\end{IEEEkeywords}

% \fi
\IEEEpeerreviewmaketitle

\section{Introduction}

The modeling of graph signals has been a pervasive topic in recent years in Graph Signal Processing (GSP) and Graph Neural Networks (GNN) \citep{Xu2019GIN,Kenlay2020StabilityPolyFilters} with a lack of a general consensus of the best underlying graph structure for modeling \citep{Ortega2018GSP,Scarselli2008GNN,Ruiz2021GNNOverview}. Often, this structure is unrelated to the graph signal itself (for example, geometric graphs for traffic signals). CoVariance Neural Networks (VNN) propose the use of the sample covariance matrix as the underlying graph shift operator (GSO)\citep{Sihag2022CoVariance}. This approach encodes pairwise relationships in a robust statistical object. Yet, while this represents relevant interactions in a \textit{static} case this does not necessarily hold when time-evolving graph signals are being modeled \citep{Li2016NonstationarySphere}.

Graph temporal convolutional neural networks (GTCNN) \citep{Isufi2021GTCNN,Sabbaqi2023GTCNN}  are a notable development in the spatio-temporal modeling of dynamically evolving graph signals. This class of models typically constructs a fully connected Cartesian or Kronecker product graph. While this effectively captures instantaneous interactions, convolutions in this domain result in a computational complexity that is quadratic in time, thus infeasible for longer time-series \citep{Leskovec2010Kronecker}. 

Given a time-evolving multi-variate signal the sample covariance represents the \textit{long-term} correlation between variables over the entire time period. However, each snap shot in time has varying \textit{instantaneous} interactions\citep{Roy2024FAST}. This difference is in fact, non-trivial. While approaches like temporal PCA \citep{Scharf2022TPCA} perform projections over the time averaged sample covariance matrix, this aggregation loses potentially useful information. This is demonstrated by the development of the time-varying graphical lasso \citep{Hallac2017TVGL}, an optimization framework that estimates a dynamic inverse covariance matrix directly from time series data. While this approach is intuitive and useful, the large computational cost of solving such an optimization problem has limited the use of this approach in neural network architectures \citep{Hamilton2017RepLearning}.

Graph Variate Signal Analysis (GVSA)\citep{Smith2019Gvsa} provides an extended general framework to GSP for the analysis of spatio-temporal signals, using general instantaneous pairwise node functions (unrestricted by matrix multiplication) to formulate data constructed dynamic graph structures. This framework motivates methods such Graph-Variate Dynamic Connectivity and FAST Functional Connectivity, where these instantaneous graphs are filtered by a stable support constructed from the long term signal coupling information of the signal itself (GVDC) or a global cohort (FAST), reducing noise in short temporal windows while providing a very high, sample by sample, temporal resolution which does not rely on a window length compared to traditional sliding window approaches.

In this work we integrate GVSA with the more traditional "convolution" aggregation found in modern Graph Neural Networks (GNNs) \citep{Li2016GGNN,Abadal2021GNNAccelSurvey,Pfrommer2021Discriminability,Isufi2024GraphFiltersTSP,Velickovic2018GAT}. For each input into the network an instantaneous connectivity tensor against a stable (and potentially learnt) support is constructed. This tensor is multiplied with its respective signal vector, this results in the capturing of spatio-temporal functional interactions. With this, we derive two important theoretical insights. Firstly, we show that while instantaneous connectivity matrices are typically rank-deficient and non-invertible, Hadamard multiplication with a full-rank stable support remedies this. Furthermore, we show that by using parallelized batch processing and low-rank matrix construction we achieve a speed up resulting in a linear time-complexity. This allows, for the first time, the capture of sample resolution signal dependent connectivity in a efficient, scalable manner.

We evaluate GVNN forecasting performance in 3 chaotic maps,2 weather forecasting tasks and 2 EEG motor imagery tasks. GVNNs successfully capture the non-trivial instantaneous temporal interactions present in multi-variable time-series. Particularly, we show that it outperforms the state of the art conventional graph based methods for time-series. Showing that the inductive bias provided by GVNNs improve performance. In application, we study EEG motor imagery classification, demonstrating that GVNNs capture the high temporal resolution of EEG signals while effectively reducing noise outperforming approaches such as EEGNet\citep{Lawhern2016EEGNet} and the Transformer model. Our results indicate that GVNNs could play a pivotal role in advancing the next generation of  Brain–Computer Interfaces (BCIs)\citep{MOABB2023,Keutayeva2024CCT,Zhang2018BCICGAN}, where minimizing calibration time and maximizing online responsiveness are crucial engineering challenges\citep{Bessadok2021GNNNetworkNeuro}.

\section{Related Work}
\subsection{Graph Neural Networks}
\begin{figure*}[t]
  \centering
  \includegraphics[width=\textwidth]{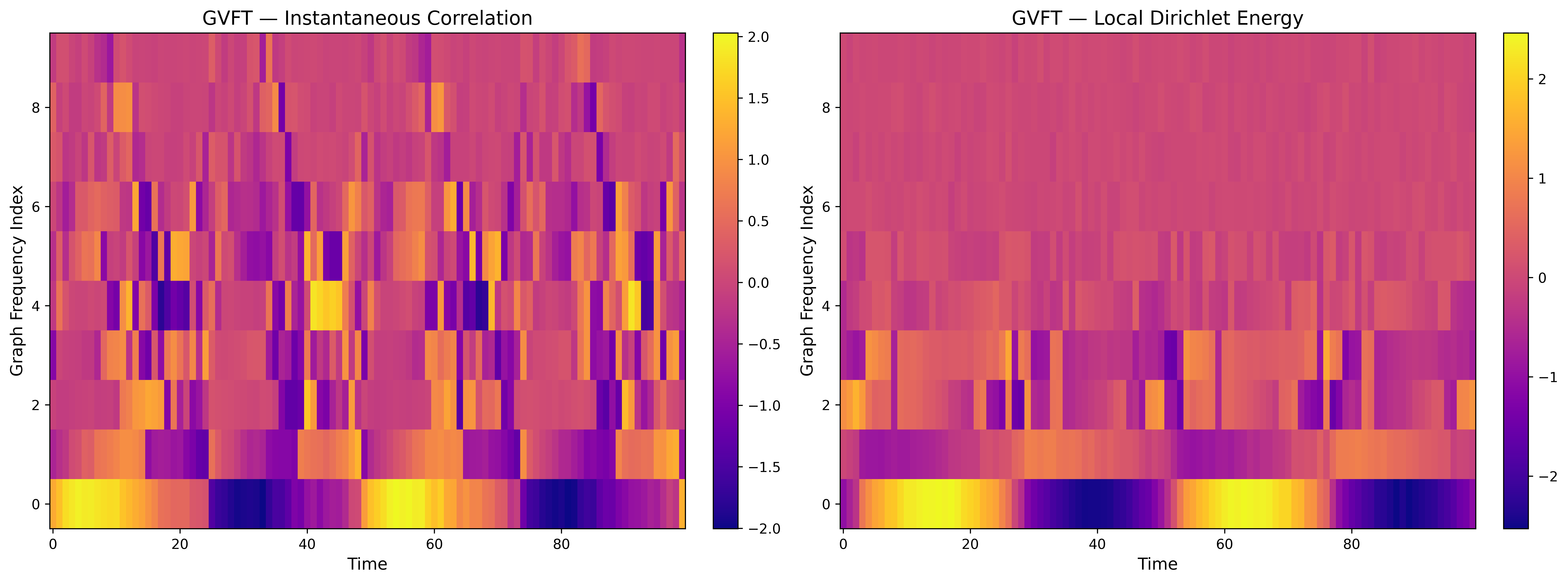}
  \caption{%
    \textbf{Graph Variate Fourier Transform (GVFT).} Each panel shows the GVFT coefficients of a synthetic multivariate time series projected onto the eigenbasis of its own graph‐structured connectivity profile at each time step. The left heatmap uses a squared‐difference formulation for 
    $\Omega_{t} = (x_{i} - x_{j})^{2} \cdot C$, while the right uses instantaneous correlation: 
    $\Omega_{t} = \mathrm{corr}(x_{t}) \cdot C$, where $C$ is the long‐term correlation matrix across the full signal. The GVFT transforms the input signal $X \in \mathbb{R}^{N \times T}$ into a new matrix 
    $\hat{X} \in \mathbb{R}^{N \times T}$, where each column represents the projection of $x_{t}$ onto the eigenbasis of $\Omega_{t}$. This figure illustrates how different formulations of signal‐derived connectivity affect the spectral content and dynamics of the transformed signal.%
  }
  \label{fig:gvft_side_by_side}
\end{figure*}
Graph Signal Processing (GSP) extends classical signal processing to data indexed by the vertices of a graph. A key component is the Graph Shift Operator (GSO), whose eigen-decomposition underlies operations analogous to the Discrete Fourier Transform (DFT). These components are the foundation on which Graph Neural Networks are built \citep{Isufi2024GraphFiltersTSP,Maskey2023GraphonTransfer,Levie2020TransferabilitySpectral}.

\begin{definition}[Graph Convolutional Filter]
\label{gfilter}
Let $\mathbf{h} = [h_0, \ldots, h_K]^\top$ be filter coefficients.  
A graph convolutional filter of order $K$ is the linear map
\begin{equation}
\mathbf{H}(\mathbf{S})\,\mathbf{x} 
= \sum_{k=0}^K h_k \,\mathbf{S}^k \mathbf{x} 
= H(\mathbf{S})\,\mathbf{x},
\end{equation}
where $\mathbf{H}(\mathbf{S}) = \sum_{k=0}^K h_k \mathbf{S}^k$.
\end{definition}

\begin{definition}[Graph Fourier Transform (GFT)]
For a diagonalizable GSO $\textbf{S} =\textbf{ V} \Lambda \textbf{V}^{-1}$ with eigenvectors $V$ and eigenvalues $\Lambda$, the GFT of a graph signal $\textbf{x}$ is 
\(
\tilde{\textbf{x}} = \textbf{V}^{-1} \textbf{x},
\)
and the inverse GFT is
\(
\textbf{x }= \textbf{V}\,\tilde{\textbf{x}}.
\)
\end{definition}

\begin{definition}[Graph Convolutional Network (GCN)]

A Graph Convolutional Network \citep{Sandryhaila2013DSPG,Zugner2019RobustGCN,Keriven2019UniversalInvariant,Hamilton2017RepLearning} layer updates a graph signal $\mathbf{X} \in \mathbb{R}^{N \times F}$ (with $N$ nodes and $F$ input features) as:
\begin{equation}
\mathbf{X}^{(\ell+1)} 
= \sigma\!\left(\mathbf{H}(\mathbf{S})\, \mathbf{X}^{(\ell)} \mathbf{W}^{(\ell)}\right),
\end{equation}
where $\mathbf{S}$ is the graph shift operator (GSO) of choice, 
$\mathbf{W}^{(\ell)} \in \mathbb{R}^{F_{\ell} \times F_{\ell+1}}$ are learnable weights, 
and $\sigma(\cdot)$ is a nonlinear activation function.  

\end{definition}

\subsection{Graph-Time Convolutional Neural Networks (GTCNNs).}

The standard way to model spatiotemporal signals is the use of product graphs to create Graph time Convolutional filters and thus GTCNN's.
\begin{definition}[Graph-Time Convolutional Neural Network (GTCNN) \citep{Isufi2021GTCNN,Sabbaqi2023GTCNN}]
Let $\mathcal{G}_{\mathcal{P}}=(\mathcal{V}_{\mathcal{P}},\mathcal{E}_{\mathcal{P}},\mathbf{S}_{\mathcal{P}})$ 
be a spatio\mbox{-}temporal product graph with shift operator 
$\mathbf{S}_{\mathcal{P}} \in \mathbb{R}^{NT \times NT}$.  
A spatio\mbox{-}temporal signal $\mathbf{X}\in\mathbb{R}^{N\times T}$ is vectorized as 
$\mathbf{x}_{\mathcal{P}} = \mathrm{vec}(\mathbf{X}) \in \mathbb{R}^{NT}$.  

The \emph{graph-time convolutional filter} of order $K$ is defined as
\begin{equation}
\mathbf{y}
= \Bigg(\sum_{k=0}^{K} h_k\, \mathbf{S}_{\mathcal{P}}^{\,k}\Bigg)\mathbf{x}_{\mathcal{P}}
\;\equiv\; \mathbf{H}(\mathbf{S}_{\mathcal{P}})\,\mathbf{x}_{\mathcal{P}},
\end{equation}
which aggregates information from $K$-hop spatio\mbox{-}temporal neighborhoods.  

For multiple features, let 
$\mathbf{X}_{\mathcal{P}}^{(\ell-1)} \in \mathbb{R}^{NT\times F_{\ell-1}}$ 
denote the input at layer $\ell-1$.  
We apply a bank of polynomial filters with coefficient matrices 
$\{\mathbf{H}_{k}^{(\ell)}\}_{k=0}^K$.  
The propagation rule of layer $\ell$ is
\begin{equation}
\mathbf{X}_{\mathcal{P}}^{(\ell)}
=\sigma\!\Bigg(
\sum_{k=0}^{K} \mathbf{S}_{\mathcal{P}}^k \,\mathbf{X}_{\mathcal{P}}^{(\ell-1)} \mathbf{H}_{k}^{(\ell)}
\Bigg),
\end{equation}
where $\mathbf{H}_{k}^{(\ell)} \in \mathbb{R}^{F_{\ell-1}\times F_\ell}$ are trainable filter coefficient matrices 
and $\sigma(\cdot)$ is a pointwise nonlinearity (e.g., ReLU).  

A $L$-layer GTCNN is obtained by stacking such modules.
\end{definition}

While effective for shorter temporal sequences the clear bottleneck here is the quadratic dependency in \textit{both} the number of nodes and sequence length, this makes the modeling of long time-series unfeasible. Furthermore the product graphs do not capture instantaneous signal specfici dependencies and are usually a binary graph.

\subsection{Graph Variate Signal Analysis}
\begin{figure*}[t]
\centering
\resizebox{\linewidth}{!}{%
\begin{tikzpicture}[line cap=round,line join=round,>=Latex,transform shape]

% ----- styles & colors -----
\tikzset{
  blk/.style={
    rounded corners=3pt,
    draw=black,
    very thick,
    minimum width=40mm,
    minimum height=9mm,
    align=center
  },
  note/.style={font=\small}
}
\definecolor{cin}{HTML}{F2B8C6}
\definecolor{cfn}{HTML}{F8E1A1}
\definecolor{csu}{HTML}{CBD9FF}
\definecolor{cgvt}{HTML}{E8D7FF}
\definecolor{cconv}{HTML}{BFE3F0}
\definecolor{ccomb}{HTML}{DDEFC5}

% ----- Left: build graph-variates from inputs -----
\node[blk, fill=cin, font=\normalsize] (x)
  {Input $X\in\mathbb{R}^{N\times T}$};

\node[blk, fill=cfn, above=7mm of x, font=\normalsize] (j)
  {Node fn.\ $J(t)=F_V\!\big(x_i(t),x_j(t)\big)$};

\node[blk, fill=csu, above=6mm of j, font=\normalsize] (w)
  {Stable support $W$};

\node[blk, fill=cgvt, above=6mm of w, font=\normalsize] (omega)
  {Graph-Variate Tensor $\Omega(t)=W\circ J(t)$};

\node[above=2mm of omega, note] {Built from inputs};

% ---- flows to create \Omega(t) ----
\draw[->, very thick] (x.north) -- (j.south);

% *** KEY: J -> Omega curved around LEFT side (like the red path) ***
\draw[->, very thick]
  ([xshift=-1mm]j.north west)
    .. controls +(-18mm,10mm) and +(-22mm,-10mm) ..
  ([yshift=-2mm]omega.west);

% W -> Omega straight
\draw[->, very thick] (w.north) -- (omega.south);

% ----- Middle/Right: GVNN layer -----
\node[blk, fill=cconv, right=33mm of x, yshift=1mm, font=\normalsize] (conv)
  {Batched Graph Convolution\\[1pt]
   $Z(t)=\Omega(t)\,X(t)$};

\node[blk, fill=ccomb, above=15mm of conv, xshift=26mm, font=\normalsize] (comb)
  {Weighted Combine\\[2pt]
   $Y(t)=\sigma\!\Big(\Theta\,[\,a_t X(t) + b_t Z(t)\,]\Big)$};

% X -> conv
\draw[->, very thick] (x.east) to[out=0, in=180] (conv.west);

% Omega -> conv
\draw[->, very thick] (omega.east)
  .. controls +(22mm,-6mm) and +(-20mm,8mm) ..
  (conv.north west);

% conv -> combine
\draw[->, very thick] (conv.north) to[out=90, in=180] (comb.west);

% skip: X -> combine
\draw[->, very thick] (x.east) .. controls +(25mm,18mm) and +(-14mm,-10mm) ..
  node[pos=.60, above, note] {skip} (comb.west);

% readout
\node[blk, fill=gray!20, right=25mm of comb, font=\normalsize] (head)
  {Readout (MLP / Task head)};
\draw[->, very thick] (comb.east) -- (head.west);

\end{tikzpicture}
}% end resizebox
\caption{Graph-Variate Neural Network (GVNN) layer. A multivariate sequence
$X\in\mathbb{R}^{N\times T}$ induces instantaneous connectivity $J(t)$, which is
combined with a long-term support $W$ to form $\Omega(t)=W\circ J(t)$. In parallel,
$X$ and $\Omega(t)$ drive a batched graph convolution $Z(t)=\Omega(t)X(t)$. A skip
connection carries $X$ to the combiner, which applies a learned linear map and a
nonlinearity, $Y(t)=\sigma\!\big(\Theta[\,a_t X(t)+b_t Z(t)\,]\big)$.}
\label{fig:gvnn-arch}
\end{figure*}

A potential issue with GSP based neural network architectures is that the relationship and relevance of the underlying graph structure to the signal is unclear and typically unchanging. There has been recent progress in addressing this in the form of CoVariance Neural Networks (VNN). Here the sample covariance matrix is used as a GSO, giving us a natural interpretation of Graph Convolution that is inherently linked to Principal Component Analysis (PCA) \citep{pca1}.

Temporal data however, is dynamic \citep{nonstat}, i.e.\ a single covariance estimation aggregating information over time may not be a suitable representation, particularly in the presence of irrelevant noise. Graph Variate Signal Analysis (GVSA) brings a sample-level, graph-weighted perspective to multivariate signals: it re-introduces node-to-node relationships in each time instant, but modulates their impact with a stable (or longer-term) graph. Importantly this \textit{does not depend on a window length}. This yields time-varying connectivity estimates and graph metrics that are more robust against momentary noise yet still capture fine-grained transient dynamics. It has been shown that GVSA outperforms many sliding-window or purely instantaneous techniques \citep{Smith2019Gvsa}.

\begin{definition}[Graph-Variate Signal Analysis]
Let $\Gamma = (\mathcal{V}, \mathbf{X}, \mathcal{E}, \mathbf{W})$ be a graph-variate signal, where 
\begin{itemize}
    \item $\mathcal{V}$ is the set of $N$ nodes,
    \item $\mathbf{X} \in \mathbb{R}^{N \times T}$ is the multivariate signal (each of the $N$ nodes has $T$ samples),
    \item $\mathcal{E}$ is the set of edges, and
    \item $\mathbf{W} \in \mathbb{R}^{N \times N}$ is the weighted adjacency matrix with entries $w_{ij}$.
\end{itemize}
Define a bivariate \emph{node-space function} $F_V$ as
\[
J_{ij}(t) 
\;=\;
F_V\bigl(x_i(t),\,x_j(t)\bigr),
\quad
\text{for } i\neq j,\quad
J_{ii}(t) = 0.
\]
\emph{Graph-Variate Signal Analysis} (GVSA) produces, at each time sample $t$, an $N \times N$ matrix given by the Hadamard (element-wise) product
\[
\boldsymbol{\Omega}(t)
\;=\;
\mathbf{W} \;\circ\; \mathbf{J}(t),
\]
whose entries are
\[
\Omega_{ij}(t)
\;=\;
\bigl[\mathbf{W} \circ \mathbf{J}(t)\bigr]_{ij}
\;=\;
w_{ij}\,F_V\!\bigl(x_i(t),\,x_j(t)\bigr).
\]
This, overall, gives an $N \times N \times T$ tensor representation.
\end{definition}

This framework not only allows a sample by sample high temporal resolution but is also computationally efficient. Note that no eigendecomposition is done at any stage and the entire analysis is in the node-space. Furthermore, node functions are typically chosen to exploit computational efficiency through low rank, vector outer product based operations. The stable support acts as an inherent stabilizer emphasizing stable long-term correlations and minimizing noise while still readily picking up instantaneous dynamics, providing a trade-off between global and local connectivity information. This is typically chose as the long-term correlation matrix of the signal itself or averaged over a cohort \citep{Roy2024FAST,Smith2019Gvsa}.

\section{Graph Variate Neural Networks}
%--------------------------------------------------------------------
\begin{figure*}[t]
  \centering
  \includegraphics[width=\textwidth]{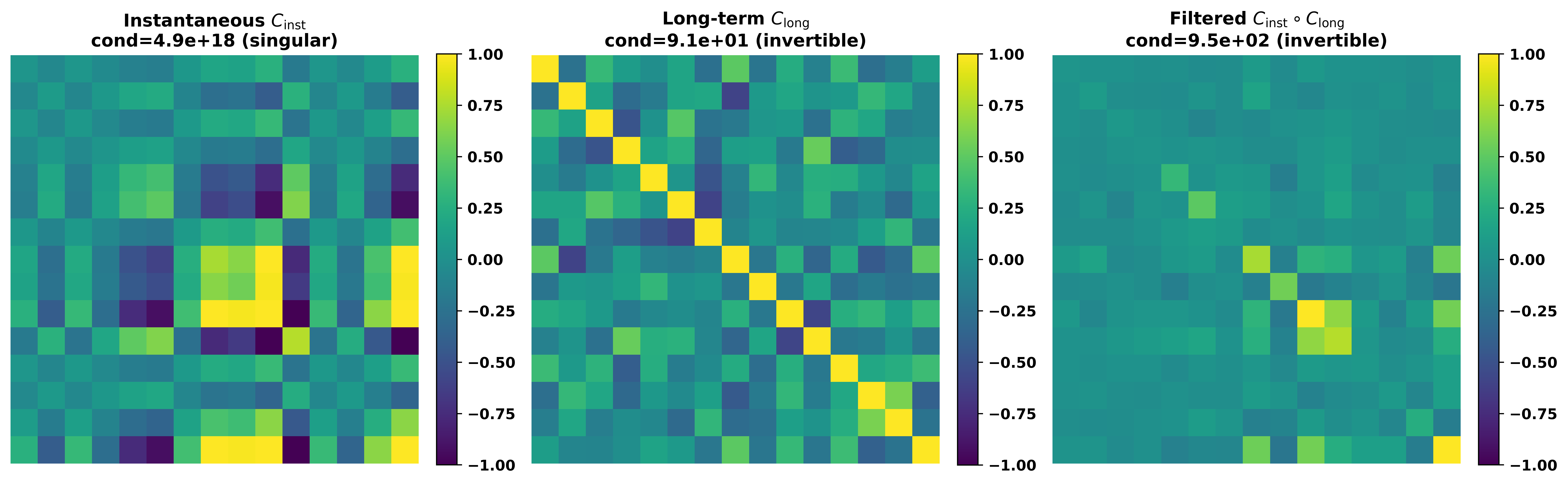}
  \caption{Comparison of instantaneous correlation profile, long-term covariance, and Hadamard-filtered covariance matrices. Each panel displays the respective matrix with its condition number and invertibility status.}
  \label{fig:correlation_matrices}
\end{figure*}

By combining GSP and GVSA approaches we conjecture that time-step wise convolution of the graph signal with its own instantaneous temporal connectivity profile can exploit the rich spatio-temporal information present in many real-life signals.

In this vein, we define Graph-Variate Neural Networks as follows.

\begin{definition}[Graph-Variate Neural Network (GVNN, layer-wise form)]
\label{def:gvnn_layer}

Let $\mathbf{W}\in\mathbb{R}^{N\times N}$ be a stable (long-term) graph support.  
For an input sequence $\mathbf{X}^{(\ell)}\in\mathbb{R}^{N\times T}$ at layer $\ell$,  
denote its $t$-th column by $\mathbf{x}^{(\ell)}(t)\in\mathbb{R}^N$.

The input-dependent graph-variate tensor is
\begin{equation}
\begin{gathered}
  \boldsymbol{\Omega}^{(\ell)}(\mathbf{X}^{(\ell)}) \in \mathbb{R}^{N\times N\times T}, \\
  \Omega^{(\ell)}_{ij}(t) = w_{ij}\,F_V\!\bigl(x^{(\ell)}_i(t),\,x^{(\ell)}_j(t)\bigr).
\end{gathered}
\end{equation}
for a chosen bivariate function $F_V(\cdot,\cdot)$.

Let $\mathbf{a}^{(\ell)},\mathbf{b}^{(\ell)}\in\mathbb{R}^T$ be learnable scalar
filter coefficients (one per time step), and let
$\mathbf{D}_{\mathbf{a}^{(\ell)}}=\operatorname{diag}(\mathbf{a}^{(\ell)})$,
$\mathbf{D}_{\mathbf{b}^{(\ell)}}=\operatorname{diag}(\mathbf{b}^{(\ell)})$.
Define the time-aligned multiplication
\begin{equation}
  \bigl(\boldsymbol{\Omega}^{(\ell)}(\mathbf{X}^{(\ell)}) * \mathbf{X}^{(\ell)}\bigr)_{:,t}
  \;=\; \boldsymbol{\Omega}^{(\ell)}(t)\,\mathbf{x}^{(\ell)}(t),
  \qquad t=1,\dots,T.
\end{equation}

Then the pre-activation output is
\begin{equation}
  \mathbf{Z}^{(\ell)}
  \;=\;
  \mathbf{X}^{(\ell)} \mathbf{D}_{\mathbf{a}^{(\ell)}}
  \;+\;
  \bigl(\boldsymbol{\Omega}^{(\ell)}(\mathbf{X}^{(\ell)}) * \mathbf{X}^{(\ell)}\bigr)\,\mathbf{D}_{\mathbf{b}^{(\ell)}},
\end{equation}

which is followed by a \emph{trainable time-mixing weight block}
$\boldsymbol{\Theta}^{(\ell)} \in \mathbb{R}^{T\times T}$ and a pointwise activation $\sigma(\cdot)$:
\begin{equation}
  \mathbf{X}^{(\ell+1)}
  \;=\;
  \sigma\!\bigl(\,\mathbf{Z}^{(\ell)} \boldsymbol{\Theta}^{(\ell)}\,\bigr)
  \;\in\;\mathbb{R}^{N\times T}.
\end{equation}

Stacking $L$ such layers yields
$\mathbf{X}^{(L)}$, which can be further mapped
to a task-dependent readout layer.
\end{definition}

Here, utilizing the fast batch based parallel processing allows a natural convolution operation where a spatio-temporal signal at a given timestep is convolved with its own connectivity profile. Also given the low rank nature of the connectivity profiles, we provide a robust platform to capture signal dependent functional inter-dependencies while being computationally efficient. Note also that we can optimize the stable support, and thus the entire dynamic connectivity profile, efficiently through training. This retains a high temporal resolution while allowing end-to-end optimization. Figure 2 provides a visual representation of a single-lay GVNN.

Equivalently, from a GSP lens, we can define the Graph-Variate Fourier Transform (GVFT) as projections of the signal vector onto its own temporal connectivity profile, this returns a matrix of size $N\times T$ that allows a simultaneous time-frequency decomposition. That is, each column of the GVFT represents the frequencies in terms of the eigenbasis of the functional graph at that time step, as demonstrated on gaussian data in figure 1.

\begin{definition}[Graph Variate Fourier Transform]
Let $\mathbf{X}=[\mathbf{x}_1,\dots,\mathbf{x}_T]\in\mathbb{R}^{N\times T}$ be a spatio-temporal signal, 
where $\mathbf{x}_t\in\mathbb{R}^N$ is the $t$-th snapshot. 
For each $t$, define
\begin{equation}
\boldsymbol{\Omega}_t = \bigl[f(x_t(i),x_t(j))\,w_{ij}\bigr]_{i,j=1}^N,
\end{equation}
with $\mathbf{W}\in\mathbb{R}^{N\times N}$ a connectivity matrix and $f(\cdot,\cdot)$ a symmetric node-pair function 
(e.g.\ $f(a,b)=(a-b)^2$).  
Since $\boldsymbol{\Omega}_t$ is symmetric, it admits $\boldsymbol{\Omega}_t = \mathbf{U}_t\boldsymbol{\Lambda}_t \mathbf{U}_t^\top$.  
The GVFT of $\mathbf{x}_t$ is
\begin{equation}
\widehat{\mathbf{x}}_t = \mathbf{U}_t^\top \mathbf{x}_t,
\end{equation}
and stacking over time yields 
$\widehat{\mathbf{X}}=[\widehat{\mathbf{x}}_1,\dots,\widehat{\mathbf{x}}_T]$.
\end{definition}

\begin{definition}[Graph-Variate frequency response]\label{def:GV_freq}
For a fixed time index $t$, let the instantaneous connectivity slice $\boldsymbol{\Omega}(t)\in\mathbb R^{N\times N}$ be symmetric with eigendecomposition
\[
\begin{aligned}
\boldsymbol{\Omega}(t) &= \mathbf{V}_t \boldsymbol{\Lambda}_t \mathbf{V}_t^{\!\top}, \\
\boldsymbol{\Lambda}_t &= \operatorname{diag}\!\bigl(\lambda_1(t), \dots, \lambda_N(t)\bigr).
\end{aligned}
\]

Consider the \emph{two-tap} Graph-Variate filter
\[
   \mathbf{y}(t) \;=\; a_t\,\mathbf{x}(t) + b_t\,\boldsymbol{\Omega}(t)\,\mathbf{x}(t),
   \qquad a_t,b_t\in\mathbb R,
\]
acting on an input vector $\mathbf{x}(t)\in\mathbb R^{N}$.
Define the Graph Fourier transforms
\[
   \tilde{\mathbf{x}}(t) := \mathbf{V}_t^{\!\top}\mathbf{x}(t),
   \qquad
   \tilde{\mathbf{y}}(t) := \mathbf{V}_t^{\!\top}\mathbf{y}(t).
\]
Substituting the eigen-decomposition yields
\[
   \tilde{\mathbf{y}}(t)
   \;=\;
   \bigl(a_t\,\mathbf{I}_N + b_t\,\boldsymbol{\Lambda}_t\bigr)\,\tilde{\mathbf{x}}(t),
\]
or component-wise,
\[
   \tilde{y}_i(t)
   \;=\;
   \underbrace{\bigl(a_t + b_t\,\lambda_i(t)\bigr)}_{h_t(\lambda_i(t))}
   \,\tilde{x}_i(t),
   \qquad i=1,\dots,N.
\]

\noindent
The scalar function
\begin{equation}
   h_t(\lambda) := a_t + b_t\,\lambda
\end{equation}
is called the \emph{instantaneous frequency response} of the
Graph-Variate filter at time~$t$.
Thus, spectrally, the filter acts as point-wise multiplication:
\begin{equation}
   \tilde{y}_i(t) \;=\; h_t(\lambda_i(t))\,\tilde{x}_i(t).
\end{equation}
\end{definition}

Definition~\ref{def:GV_freq} is in direct analogy with the classical convolution theorem
$\tilde{y}_i = \tilde{h}(\lambda_i)\,\tilde{x}_i$ for polynomial graph
filters, but with a spectrum $\{\lambda_i(t)\}$ and a response
$h_t$ that are re-evaluated at every time step.

\begin{theorem}[Parseval identity for the GVFT]
\label{thm:GVFTParseval}
For every time index $t$ and every signal vector $\mathbf{x}(t)$,
\[
\sum_{i=1}^{N} \bigl|\widehat{x}_i(t)\bigr|^{2} = \sum_{i=1}^{N} \bigl|x_i(t)\bigr|^{2}.
\]
Equivalently, $\|\widehat{\mathbf{x}}(t)\|_{2} = \|\mathbf{x}(t)\|_{2}$.
\end{theorem}

\begin{proof}
As long as the stable support is symmetric, the eigendecomposition of a connectivity profile results in $\mathbf{U}_t$ being orthonormal, i.e., $\mathbf{U}_t^{\!\top}\mathbf{U}_t = \mathbf{I}_{N}$. Applying this to $\widehat{\mathbf{x}}(t) = \mathbf{U}_t^{\!\top}\mathbf{x}(t)$, we get:
\begin{align*}
\|\widehat{\mathbf{x}}(t)\|_{2}^{2}
&= \widehat{\mathbf{x}}(t)^{\!\top}\widehat{\mathbf{x}}(t) \\
&= (\mathbf{U}_t^{\!\top}\mathbf{x}(t))^{\!\top}(\mathbf{U}_t^{\!\top}\mathbf{x}(t)) \\
&= \mathbf{x}(t)^{\!\top} \mathbf{U}_t \mathbf{U}_t^{\!\top} \mathbf{x}(t) \\
&= \mathbf{x}(t)^{\!\top} \mathbf{x}(t) = \|\mathbf{x}(t)\|_{2}^{2}.
\end{align*}
\end{proof}

\begin{remark}
Because the GVFT basis $\mathbf{U}_t$ depends on the instantaneous, \emph{signal-derived} slice $\boldsymbol{\Omega}(t)$, Parseval's identity holds \emph{separately} for each time step $t$; summing over $t$ shows energy conservation for the entire spatio-temporal matrix $\mathbf{X}=[\mathbf{x}(1)\dots \mathbf{x}(T)]$:
\[
   \sum_{t=1}^{T}\|\widehat{\mathbf{x}}(t)\|_{2}^{2}
   =\sum_{t=1}^{T}\|\mathbf{x}(t)\|_{2}^{2}.
\]

\end{remark}

While we can clearly extend GVNNs by including higher order polynomials per time-step, we exclude these for the sake of simplicity. We further note that, computationally (and intuitively), a right multiplication with a time-wise filter coefficient matrix is more efficient than using polynomial filter coefficients (the typical choice in the GNN literature).

This dual perspective is a shift from the traditional GSP sense of graph frequencies given that the graph is constructed from the signal itself. In fact there is a closer link to PCA present here. Projecting signals onto a data-driven dynamic eigenbasis (i.e.\ the sample covariance matrix in PCA), supported by a stable support, allows a high level of precision and interpretability.

\subsection{Temporal Signal Dependent Convolution}

\begin{table*}[t]
\centering
\caption{MSE For Temporal Graph Models on different chaotic maps}
\label{tab:chaotic_mse}
\begin{tabular}{llccc}
\toprule
Dataset & Model & $H{=}3$ & $H{=}6$ & $H{=}12$ \\
\midrule
\multirow{4}{*}{Hopfield} 
& GVNN   & $\mathbf{0.0237 \pm 0.0008}$ & $\mathbf{0.1131 \pm 0.0024}$ & $\mathbf{0.1837 \pm 0.0053}$ \\
& GTCNN  & $0.1029 \pm 0.0052$ & $0.1683 \pm 0.0014$ & $0.2229 \pm 0.0031$ \\
& GVARMA & $0.5283 \pm 0.0082$ & $0.5846 \pm 0.0086$ & $0.6514 \pm 0.0060$ \\
& GGRNN  & $0.0628 \pm 0.0166$ & $0.1742 \pm 0.0083$ & $0.2662 \pm 0.0107$ \\
\midrule
\multirow{4}{*}{Lorenz} 
& GVNN   & $\mathbf{0.2143 \pm 0.0083}$ & $\mathbf{0.5001 \pm 0.1623}$ & $0.7325 \pm 0.0092$ \\
& GTCNN  & $0.8163 \pm 0.0456$ & $0.8595 \pm 0.0282$ & $0.9039 \pm 0.0145$ \\
& GVARMA & $0.8739 \pm 0.0188$ & $0.8764 \pm 0.0397$ & $0.9027 \pm 0.0027$ \\
& GGRNN  & $0.3528 \pm 0.0271$ & $0.5327 \pm 0.0159$ & $\mathbf{0.5971 \pm 0.0049}$ \\
\midrule
\multirow{4}{*}{MacArthur} 
& GVNN   & $\mathbf{0.0910 \pm 0.0004}$ & $\mathbf{0.2509 \pm 0.0046}$ & $\mathbf{0.3914 \pm 0.0087}$ \\
& GTCNN  & $0.8800 \pm 0.0148$ & $0.8479 \pm 0.0123$ & $0.8856 \pm 0.0015$ \\
& GVARMA & $0.5454 \pm 0.0325$ & $0.7608 \pm 0.0794$ & $0.8355 \pm 0.0212$ \\
& GGRNN  & $0.2232 \pm 0.0009$ & $0.4252 \pm 0.0099$ & $0.5073 \pm 0.0034$ \\
\bottomrule
\end{tabular}
\end{table*}
Temporal information provides rich, discriminative information that could significantly enhance machine learning models. As an example, EEG signals have a very high temporal resolution. While traditionally being studied in the frequency or spectral domain, time domain analysis of EEG signals provides great potential in enhancing Brain Computer Interfaces (BCI).

We focus here on two common temporal domain node-space functions. Given graph signals $\mathbf{x}$, we define:

\begin{itemize}
\item \textbf{Instantaneous Correlation (IC):}
\begin{equation}
F_V(x_i(t), x_j(t)) 
= \left| \big(x_i(t) - \bar{x}_i\big)\,\big(x_j(t) - \bar{x}_j\big) \right|,
\end{equation}
where $\bar{x}_i = \tfrac{1}{T}\sum_{t=1}^T x_i(t)$ is the temporal mean of node $i$.

\item \textbf{Local Dirichlet Energy (LDE) \citep{mde}:}
\begin{equation}
F_V(x_i(t), x_j(t)) = \big(x_i(t)-x_j(t)\big)^2.
\end{equation}
\end{itemize}

Instantaneous correlation is rank-1 and LDE rank-3, both expressible as sums of outer products. Such structures are efficient, as outer products reduce to parallelizable vector operations that GPUs compute rapidly. This approach combined with the Hadamard support is inspired by recent advances in Parameter-Efficient Fine-Tuning (PEFT) \citep{hu2021lora,huang2025hira}, a method to improve the efficiency of Large Language Models (LLMs). We direct the interested reader to Appendix~\ref{lorahira}.

The instantaneous correlation captures co-deviation from mean temporal patterns. The LDE node function has a direct relationship to the Dirichlet energy and captures local node gradient changes. We can also take linear combinations of these two node functions in order to exploit both their contrasting views.

There is an important observation to make here with transformers, given that the attention mechanism can be argued to use a ``graph'' constructed from the data itself for convolution \citep{att}. In fact, recent ideas have provided a unifying view on Transformers and GNNs, arguing that transformers are GNNs that operate on a data-specific graph \citep{joshi2025}. Thus GVNNs can be argued to be a form of attention with a fundamentally different formulation, i.e.\ the time-step specific tensor weighted against a stable support. We expand on this in Appendix~\ref{sec:theory-data-driven-graphs} for the interested reader.

We first introduce the following definitions.

\textbf{Definition 1.}  
We observe $T$ time-centered samples $\mathbf{x}_t \in \mathbb{R}^N$ for $t = 1, \dots, T$, and define
\[
  \bar{\mathbf{x}}= \frac{1}{T} \sum_{t=1}^{T} \mathbf{x}_t,
  \quad
  \tilde{\mathbf{x}}_t= \mathbf{x}_t - \bar{\mathbf{x}}
\]
and we assume
\[
  \boxed{\mathbf{W}\succ 0} \quad \text{(PSD)}.
\]
Let $w_{ij} = W_{ij}$. For each fixed $t$, define the stabilized \emph{instantaneous correlation profile}
\[
  \rho_t(i,j)
  = w_{ij} \, \bigl| \tilde{x}_i^{(m)}(t) \, \tilde{x}_j^{(m)}(t) \bigr|,
  \quad i,j = 1, \dots, N.
\]

\vspace{1em}

\textbf{Definition 2 (Sylvester's Law of Inertia).}  
Let $\mathbf{A} \in \mathbb{S}^N$ be a symmetric matrix of rank $r$ with inertia $(p, q, 0)$, meaning $p$ positive and $q$ negative eigenvalues such that $p + q = r$. Then $\mathbf{A}$ is congruent to the diagonal normal form
\[
  \mathbf{G} =
  \begin{pmatrix}
    \mathbf{I}_p & \mathbf{0} & \mathbf{0} \\
    \mathbf{0} & -\mathbf{I}_q & \mathbf{0} \\
    \mathbf{0} & \mathbf{0} & \mathbf{0}
  \end{pmatrix},
  \quad p + q = r.
\]
Two symmetric matrices are congruent if and only if they have the same rank and signature $(p, q, 0)$.

\label{sec:stat_no_z}
\label{ssec:schur_full_rank}

\begin{theorem}[Full-rank preservation under Hadamard filtering]
\label{thm:schur_full_rank}
Let $J_{ij}=F_V(x_i(t), x_j(t)) = x_i(t)x_j(t)$ be the \textit{unfiltered} instantaneous correlation profile with rank $m<N$. 
If every component of $\tilde{\mathbf{x}}_t^{(m)}$ is non-zero, and $\mathbf{W}$ is of full rank, then
\[
  \operatorname{rank}\bigl(\boldsymbol{\Omega}(t)
\;=\;
\mathbf{W} \;\circ\; \mathbf{J}(t)\bigr) = N,
  \quad\text{i.e.\ }\boldsymbol{\Omega}(t)\text{ is invertible.}
\]
Moreover, $\boldsymbol{\Omega}(t)$ is symmetric positive-definite, preserving the signature of $\mathbf{W}$.
\end{theorem}

\begin{proof}[Proof of Theorem~\ref{thm:schur_full_rank}]
Set
\[
  \mathbf{d}_t := \bigl|\tilde{\mathbf{x}}_t\bigr|\in\mathbb{R}^N,
  \qquad
  \mathbf{D}_t:= \diag\bigl(\mathbf{d}_t\bigr),
\]
so each $\mathbf{D}_t^{(m)}$ is diagonal with strictly positive entries and thus invertible. The Hadamard product identity gives
\[
  \boldsymbol{\Omega}(t)
  = \mathbf{W}\circ \bigl(\tilde{\mathbf{x}}_t\tilde{\mathbf{x}}_t^{(m)\!\top}\bigr)
  = \mathbf{D}_t\,\mathbf{W}\,\mathbf{D}_t,
\]
i.e.\ $\boldsymbol{\Omega}(t)$ is congruent to $\mathbf{W}$.

Now applying Sylvester's Law, since $\mathbf{W}\succ0$ has inertia $(N,0,0)$, any matrix congruent to it must share the same inertia. Therefore
\[
 \boldsymbol{\Omega}(t)\succ0,
  \quad
  \mathrm{rank}(\boldsymbol{\Omega}(t))=\mathrm{rank}(\mathbf{W})=N.
\]
This establishes both invertibility and positive-definiteness.
\end{proof}

This theorem shows that Hadamard filtration with a stable support \textit{induces} stability into the instantaneous correlation profile.

Figure~\ref{fig:correlation_matrices} shows empirical evidence of Theorem~\ref{thm:schur_full_rank} where the Hadamard filtered matrix by the full-rank long-term correlation matrix is now invertible and has a much lower condition number. We prove similar results for the LDE case in Appendix~\ref{rankliftlde}.

The LDE connectivity profile has a distinct relationship with the traditional Dirichlet energy of a signal (naturally encoding a measure of smoothness into signal convolutions) as shown in the following theorem.

\begin{theorem}[Gershgorin--Dirichlet Bound]
\label{thm:gershgorin_dirichlet}
Let $\mathbf{W}\in\mathbb R^{N\times N}$ be symmetric and $\mathbf{x}\in\mathbb R^N$ any signal. Form
\[
J_{ij}(t)=F_V(x_i(t), x_j(t))= (x_i(t)-x_j(t))^2,
\quad \boldsymbol{\Omega}(t)=\mathbf{W}\circ \mathbf{J}(t),
\]
and define
\[
\mathcal{E}_{\mathrm{abs}}(t)=\tfrac{1}{2}\sum_{i,j}|w_{ij}\,(x_i(t)-x_j(t))^2|.
\]
Then the spectral radius satisfies
\[
\rho(\boldsymbol{\Omega}(t))\le 2\,\mathcal{E}_{\mathrm{abs}}(t).
\]
\end{theorem}
\begin{table*}[t]
\centering
\small
\caption{Final test MSE (lower is better) for PEMS-BAY and METR-LA across all models.}
\label{tab:pems_metr_all}
\begin{tabular}{llccc}
\toprule
Dataset & Model & Horizon 3 & Horizon 6 & Horizon 12 \\
\midrule
\multirow{7}{*}{PEMS-BAY}
  & GVNN (Trainable $\mathbf{W}$) & $\mathbf{0.1722 \pm 0.0093}$ & $\mathbf{0.2323 \pm 0.0080}$ & $\mathbf{0.3250 \pm 0.0229}$ \\
  & Transformer          & $0.3126 \pm 0.0099$          & $0.3467 \pm 0.0026$          & $0.3858 \pm 0.0061$          \\
  & LSTM                 & $0.3686 \pm 0.0231$          & $0.3810 \pm 0.0085$          & $0.4058 \pm 0.0022$          \\
  & GVNN (Static $\mathbf{W}$) & $0.7017 \pm 0.0460$          & $0.7642 \pm 0.0611$          & $0.8097 \pm 0.0280$          \\
  & GTCNN                & $0.9703 \pm 0.0032$          & $1.0010 \pm 0.0099$          & $1.0474 \pm 0.0071$          \\
  & GVARMA               & $0.7940 \pm 0.0128$          & $0.8271 \pm 0.0113$          & $0.8862 \pm 0.0052$          \\
  & GGRNN                & $0.8766 \pm 0.0040$          & $0.9175 \pm 0.0061$          & $0.9736 \pm 0.0018$          \\
\midrule
\multirow{7}{*}{METR-LA}
  & GVNN (Trainable $\mathbf{W}$) & $\mathbf{0.2218 \pm 0.0017}$ & $\mathbf{0.3082 \pm 0.0158}$ & $\mathbf{0.4434 \pm 0.0033}$ \\
  & Transformer          & $0.2928 \pm 0.0104$          & $0.3799 \pm 0.0072$          & $0.5384 \pm 0.0214$          \\
  & LSTM                 & $0.3554 \pm 0.0054$          & $0.4355 \pm 0.0021$          & $0.6644 \pm 0.0280$          \\
  & GVNN (Static $\mathbf{W}$) & $0.6012 \pm 0.0625$          & $0.6631 \pm 0.0790$          & $0.7076 \pm 0.0301$          \\
  & CPGraphST            & $0.9082 \pm 0.0191$          & $0.9234 \pm 0.0211$          & $0.9887 \pm 0.0138$          \\
  & GVARMA               & $0.9713 \pm 0.0364$          & $0.9527 \pm 0.0447$          & $1.0680 \pm 0.0339$          \\
  & GGRNN                & $0.8205 \pm 0.0167$          & $0.8621 \pm 0.0089$          & $0.9281 \pm 0.0048$          \\
\bottomrule
\end{tabular}
\end{table*}

\label{sec:theorem2proof}

\begin{proof}[Proof of Theorem~\ref{thm:gershgorin_dirichlet}]
Recall Gershgorin's circle theorem: if $\mathbf{A}=(a_{ij})$ is any $N\times N$ matrix then each eigenvalue $\lambda$ of $\mathbf{A}$ satisfies
\[
\lambda\in D\!\bigl(a_{ii},\,R_i(\mathbf{A})\bigr)
\quad\text{where}\quad
R_i(\mathbf{A})=\sum_{j\ne i}|a_{ij}|.
\]
In our case $\Omega_{ii}(t)=0$ and
\[
R_i(\boldsymbol{\Omega}(t))
=\sum_{j\ne i}|\Omega_{ij}(t)|
=\sum_{j\ne i}|w_{ij}\,(x_i(t)-x_j(t))^2|,
\]
so every eigenvalue $\delta$ of $\boldsymbol{\Omega}$ lies in one of the real intervals $[-R_i,R_i]$. Taking the union over $i$ gives
\[
\rho(\boldsymbol{\Omega})\subset\bigcup_{i=1}^N[-R_i,R_i]
=\bigl[-\max_iR_i,\;\max_iR_i\bigr].
\]

By definition,
\[
R_i
=\sum_{j\ne i}|w_{ij}\,(x_i(t)-x_j(t))^2|.
\]
Summing these radii over all $i$ yields
\begin{align*}
\sum_{i=1}^N R_i
&= \sum_{i=1}^N \sum_{j \ne i} \bigl|w_{ij}\,(x_i(t)- x_j(t))^2\bigr| \\
&= \sum_{i,j} \bigl|w_{ij}\,(x_i(t) - x_j(t))^2\bigr| \\
&= 2\,\mathcal{E}_{\mathrm{abs}}.
\end{align*}
Thus the total ``Gershgorin mass'' equals twice the Dirichlet energy.

Since $\rho(\boldsymbol{\Omega})=\max|\delta|\le\max_iR_i$, we need only show
$\max_iR_i\le2\mathcal{E}_{\mathrm{abs}}$. But from above,
$\sum_iR_i=2\mathcal{E}_{\mathrm{abs}}$, and the largest term in a sum of nonnegative numbers is no bigger than the sum itself. Hence
\[
\max_iR_i \;\le\; \sum_iR_i = 2\,\mathcal{E}_{\mathrm{abs}},
\]
hence
\[
\rho(\boldsymbol{\Omega})\le2\,\mathcal{E}_{\mathrm{abs}},
\]
completing the proof.
\end{proof}

Theorem~\ref{thm:gershgorin_dirichlet} shows that the spectral radius of the Hadamard filtered LDE is upper bounded by twice the absolute Dirichlet energy of the signal on the stable support. Intuitively, this ensures that the GVNN convolution is \textit{smoothness-aware} (see Appendix for more details). This relates the spectral radius of the LDE connectivity profile with the traditional Dirichlet energy of a graph signal on the stable support $\mathbf{W}$.

\section{Experimental Results}

%\subsection{GVNNs in predicting chaotic behaviour}
%\begin{figure*}[t]
  %\centering
  %\includegraphics[width=\textwidth]{Chaos.png}
 % \caption{
  %  \textbf{3D PCA trajectories of Mac Arthur, Hopfield, and Coupled Lorenz systems.}}

%  \label{fig:3d_trajectories}
%\end{figure*}

\subsection{Chaotic Maps}

\begin{figure*}[t]
  \centering
  \includegraphics[width=\textwidth]{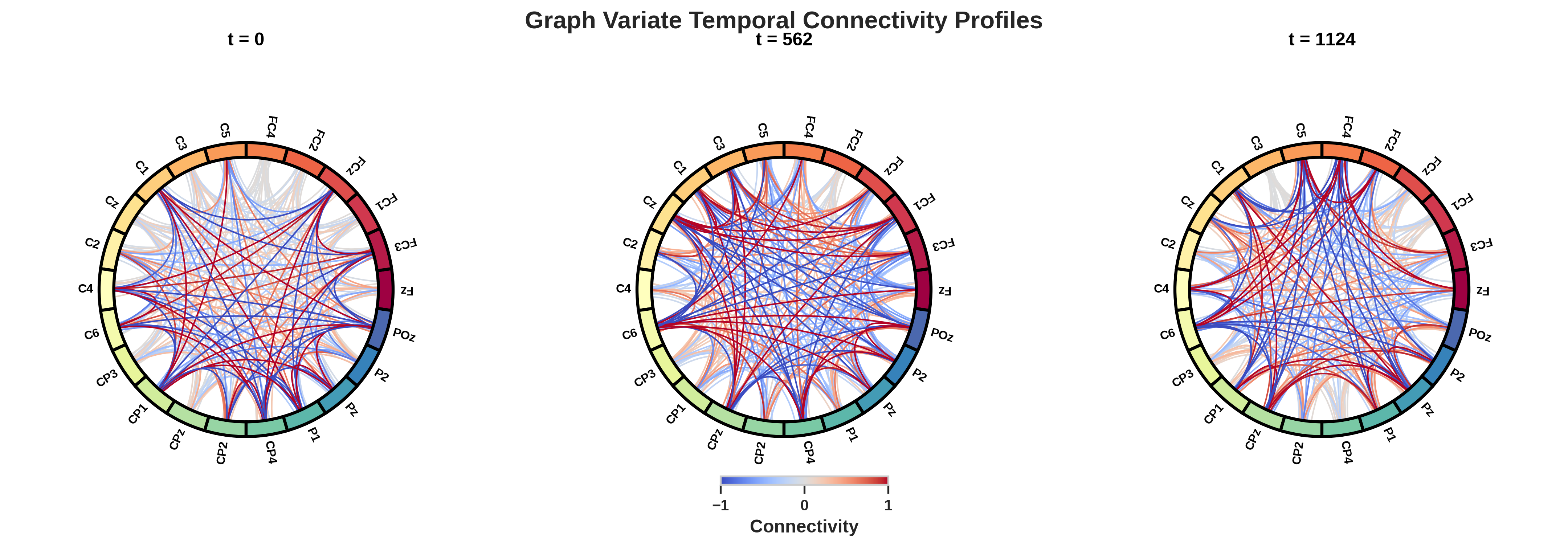}
  \caption{\textbf{Instantaneous Connectivity Profiles (Inst.\ Correlation). Transient Brain states are not stationary. This figure shows how, depending on the temporal position in the time-locked task, connectivity can change significantly. GVNNs exploit this, allowing a framework for the analysis of \textit{Dynamic Functional Connectivity.}}}
    \label{fig:NetC}
\end{figure*}

Chaotic systems pose unique challenges to statistical learning models and are also interpretable as benchmarks.  They thus provide a baseline to compare GVNN's with other graph based models for time-series \citep{chaos}.

We compare GVNNs with a standard GTCNN, a Gated Graph RNN (GGRNN) and Graph VARMA (GVARMA) model. For our node function we used a linear combination of the local dirchlet energy and instantaneous correlation while allowing the stable support to be learnt from data.

We have chosen these models primarily due to their core operation being some sort of Graph Convolution. Note we are not considering hybrid models such as Graph Wavenet \citep{} however do foresee future work incorporating GVNNs into hybrid architectures. 
We have chosen these models primarily due to their core operation being some sort of Graph Convolution. Note we are not considering hybrid models such as Graph Wavenet \citep{Wu2019GraphWaveNet} however we do foresee future work incorporating GVNNs into hybrid architectures. 

For all models except GTCNN (Which uses the long-term correlation as the spatial component for fairness) we initialize the stable support with the long-term stable correlation of the chaotic map and let the model optimize this end-to-end. The node function was a linear combination of the LDE and instantaneous correlation.  

We evaluate three multi-dimensional chaotic maps. The Coupled Lorenz, Hopfield and MacArthur maps
 
We see that GVNNs perform the best over all horizons on the Hopfield and Macarthur Map with large gains being visible in the MacArthur dataset in particular.
In the coupled  Lorenz map, while GVNNs perform the best over horizons of length $1$ and $3$, they are outperformed by GGRNNs over the horizon of length $5$. This could be due to temporal interactions being less predictive for longer horizon in this chaotic map, further, a model incorporating a combination of GVNNs and GGRNNs may be promising.

\begin{table*}[ht]
\centering
\small
\caption{EEG classification: Overall summary (K-fold CV)}
\label{tab:eeg_results}
\begin{tabular}{llccc}
\toprule
Dataset & Model & Accuracy (\%) & Kappa & Time (s) \\
\midrule
\multirow{4}{*}{BCI-2A}
& GVNN (LDE + Static W)       & $60.15 \pm 1.21$ & $0.4686 \pm 0.0162$ & 0.5 \\
& \textbf{EEGNet}             & $\mathbf{60.51 \pm 3.88}$ & $\mathbf{0.4735 \pm 0.0517}$ & 1.0 \\
& Transformer                 & $51.99 \pm 3.01$ & $0.3598 \pm 0.0401$ & 1.5 \\
& LSTM                        & $52.76 \pm 2.27$ & $0.3701 \pm 0.0303$ & 1.5 \\
\midrule
\multirow{4}{*}{PhysioNet}
& GraphVar+MLP (LDE + Learned W) & $80.29 \pm 0.82$ & $0.6058 \pm 0.0164$ & 2.0 \\
& \textbf{Transformer}           & $\mathbf{80.94 \pm 0.87}$ & $\mathbf{0.6189 \pm 0.0173}$ & 0.9 \\
& LSTM                           & $74.19 \pm 1.74$ & $0.4834 \pm 0.0351$ & 1.4 \\
& EEGNet                         & $79.61 \pm 1.55$ & $0.5922 \pm 0.0310$ & 3.2 \\
\bottomrule
\end{tabular}
\end{table*}

\subsection{Traffic Forecasting}

We evaluate four graph-based forecasting models on the METR-LA and PEMS-BAY traffic networks, two widely adopted benchmarks for spatio-temporal prediction tasks. We also compare performance with the more commonly used transformer and LSTM models, which serve as strong non-graph baselines. METR-LA contains four months of speed measurements from 207 sensors distributed across Los Angeles County, recorded at 5 minute intervals. PEMS-BAY comprises six months of data from 325 sensors in the San Francisco Bay Area at the same temporal resolution~\citep{Sun2020TrafficGLT,trafficforecasting}. Both datasets exhibit complex spatial dependencies arising from road network topology and temporal patterns driven by rush hours, weekends, and other recurring traffic phenomena.

Following standard practice in the traffic forecasting literature, we predict future speeds at horizons $h\in\{3,6,12\}$ time-steps (i.e.\ 15, 30, and 60 minutes ahead) given the past $T=6$ observations (30 minutes of historical data). The graph-based models follow the same architectural layout as in the previous experiment. However, for this evaluation we consider two variants of the two-layer GVNN: one using a fixed adjacency structure derived from the road network, and another where the support $W$ is treated as a trainable parameter that can be optimized during training.

Table 2 presents our results across both datasets and all prediction horizons. It can be noted that using a fixed support, GVNNs outperform the other graph-based models but remain inferior to the LSTM and Transformer baselines. Allowing $W$ to be learned, however, results in substantial gains in performance, with GVNNs significantly outperforming all competing models across nearly all settings. This suggests that the predefined road network adjacency, while informative, does not fully capture the effective dependencies between sensor locations---learned connections allow the model to discover latent relationships that better reflect actual traffic flow patterns. As these datasets have a large number of nodes, we do observe that GVNNs incur a notable increase in training time compared to the simpler baselines. Nevertheless, we believe that the considerable improvement in predictive accuracy justifies this computational overhead, particularly in applications where forecast quality is paramount.

% Requires in preamble:
% \usepackage{booktabs}

%\begin{figure*}[!htbp]
 %   \centering
  %  \includegraphics[width=0.9\linewidth]{Trained.png}
   % \caption{\textbf{Learned graph support matrix $W_C$ before and after training.} The figure illustrates how the static graph support matrix $W_C$ evolves through training. The left panel shows the initialized matrix, while the right panel presents the learned weights after optimization, revealing how the model adapts graph connectivity structure for improved forecasting.}
    %\label{fig:learned_W_C}
%\end{figure*}

\subsection{EEG Motor Imagery Tasks}

The BCI Competition IV 2a dataset~\citep{MOABB2023} comprises EEG recordings from nine subjects performing four motor imagery tasks (left hand, right hand, feet, tongue), with data recorded using a 17 channel setup. This dataset presents a challenging four-class classification problem that has become a standard benchmark for evaluating motor imagery decoding algorithms. The PhysioNet dataset comprises EEG recordings from 109 healthy subjects performing a simpler binary classification task, where participants imagine opening and closing either their right or left fist. The data is recorded using a 64 channel setup, providing substantially higher spatial resolution but also increasing the dimensionality of the input.

We evaluate all models using 5-fold cross validation with independent data splits to ensure robust performance estimates. For the BCI-2A dataset, we use a fixed $W$ set as the global long-term correlation matrix computed from the training set, leveraging the inherent spatial relationships between EEG channels. For the PhysioNet task, we instead allow $W$ to be learned during training, enabling the model to discover task-relevant connectivity patterns from the higher-dimensional channel space.

Table~\ref{tab:eeg_results} presents our results across both datasets. As expected, GVNNs exhibit faster training speeds on the lower channel BCI-2A dataset and achieve the fastest training time overall, with EEGNet only outperforming it slightly in terms of accuracy while requiring twice the computation time. Both models significantly surpass the LSTM and Transformer baselines on this dataset, suggesting that architectures designed with EEG structure in mind offer clear advantages for motor imagery decoding.

For the PhysioNet dataset, we observe an increase in training time for the GVNN model given the increase in channel count from 17 to 64. Despite this additional computational cost, GVNNs achieve competitive performance with the Transformer model while outperforming EEGNet in both accuracy and training speed. These results demonstrate that GVNNs scale reasonably well to higher-dimensional EEG setups while maintaining strong predictive performance.

\section{Conclusion and Limitations}

In this work we have introduced Graph Variate Neural Networks- a general framework that constructs signal dependent dynamic graph structures in a computationally efficient manner by exploiting one-shot batch processing. We further introduced two interpretable node functions, the Local Dirichlet Energy and instantaneous correlation. We show theoretically how a stable support can 'stabilize' these low-rank instantaneous structures while also being computationally simple.

In the notoriously hard task of EEG motor imagery classification, we show that GVNN's are competitive with and sometimes outperform (in terms of efficiency) traditionally used models such as the Transformer Architecture or EEGNet. This improvement in performance was sustained in the forecasting of chaotic systems, where non-trivial instantaneous interactions are present. GVNN's retained their superiority in traffic forecasting tasks, strongly outperforming strong traditional and graph based baselines.

We note that while we effectively capture \textit{intra}-channel connectivity, we are disregarding auto-correlative behavior by not connecting nodes in the time dimension. However, the improvement in performance by including signal dependent graph structures and reduction in computational time justify this decision. Furthermore,  mechanisms such as a temporal attention or convolutional module can be applied right after a GVNN layer to attend to inter time-dependencies.

We also note that our approach retains the quadratic complexity with the number of nodes such as in GTCNNs. This can become large when constructing signal specific connectivity profiles; however, such an approach \textit{would not be possible} using a product graph. Further work should also develop new node functions and stable supports, potentially incorporating spatial properties or even information theoretic measures.

\bibliography{refernces}

\appendices
\section{Appendix}
%--------------------------------------------------------------------
% Requires in preamble: \usepackage{booktabs,multirow,makecell}
% Requires in preamble: \usepackage{booktabs,multirow,makecell}
\subsection{Hardware}
All experiments were run on a single \textbf{NVIDIA A100} GPU (40\,GB VRAM). Training used \textbf{FP32} precision (no mixed precision), and all runs were executed on a single device without model or data parallelism.

\subsection{Experimental Details: Chaotic Maps}

We consider three standard discrete-time chaotic benchmarks\citep{chaos}:
\emph{Coupled Lorenz}: a network of Lorenz oscillators with diffusive coupling between state variables, producing high-dimensional, synchronized--desynchronized regimes;
\emph{Hopfield map}: a discrete-time Hopfield network with frustrated connectivity (competing attractors) yielding complex transient dynamics;
\emph{MacArthur map}: a discrete-time ecological competition model (species competing for shared resources) exhibiting multi-species chaotic population fluctuations.
Each dataset provides multivariate sequences \(X\in\mathbb{R}^{N\times T}\) (channels \(=N\) nodes).

We use a sliding window of length \(T{=}3\) to forecast horizons \(H\in\{1,3,5\}\) (one-, three-, and five-step ahead).
Windows slide with stride 1. Data are split chronologically into \(80\%\) train\(+\)val and \(20\%\) test; within the first \(80\%\) we take \(80\%\) train and \(20\%\) validation.
Inputs are z-scored per sample across channels.
All graph-based models use a \emph{trainable} support \(W_C\) initialized from the long-term channel wise Pearson correlation over the \emph{training} split , and fuse instantaneous operators by Hadamard product; dynamic slices are re-normalized as \(D^{-\frac12}(A{+}I)D^{-\frac12}\).
All models use 1 convolution layer with the GVNN using a linear combination of the two node functions. We treat GTCNN as a simple baseline with it's spatial component being the \textbf{fixed} long-term correlation matrix and the rest of the models allow end to end training of the graph.
All models consist of a MLP readout layer with Leaky ReLU activation.

We train for 500 epochs with Adam (MSE loss), batch size 128, and report the best-validation checkpoint on test.
Unless otherwise stated, we use three seeds \(\{124,14,124235\}\) and hidden dimension 128.

%-----------------------------------------------
% Chaotic maps: dataset-level hyperparameters
%-----------------------------------------------
\begin{table*}[!t]
\centering
\small
\caption{Chaotic maps forecasting: dataset-level hyperparameters (identical across maps).}
\label{tab:chaos-dataset-hparams}
\begin{tabular*}{\textwidth}{@{\extracolsep{\fill}}lcccccccc@{}}
\toprule
\textbf{Map} & \textbf{T} & \textbf{H} & \textbf{Stride} & \textbf{Split} & \textbf{Batch} & \textbf{Epochs} & \textbf{Seeds} & \textbf{Norm} \\
\midrule
Coupled Lorenz & 3 & 1, 3, 5 & 1 & 80/20 (chron.) & 128 & 500 & 124, 14, 124235 & per-sample z-score (channels) \\
Hopfield       & 3 & 1, 3, 5 & 1 & 80/20 (chron.) & 128 & 500 & 124, 14, 124235 & per-sample z-score (channels) \\
MacArthur      & 3 & 1, 3, 5 & 1 & 80/20 (chron.) & 128 & 500 & 124, 14, 124235 & per-sample z-score (channels) \\
\bottomrule
\end{tabular*}
\end{table*}

%---------------------------------------------------------
% Chaotic maps: model hyperparameters and operator details
%---------------------------------------------------------
\begin{table*}[!t]
\centering
\small
\caption{Model hyperparameters and operator details for chaotic maps (all are graph-based; $W_C$ is \emph{trainable} and initialized from long-term correlation).}
\label{tab:chaos-model-hparams}
\begin{tabular*}{\textwidth}{@{\extracolsep{\fill}}lcccc@{}}
\toprule
\textbf{Model} & \textbf{LR} & \textbf{Hidden} & \textbf{Epochs} & \textbf{Trainable $W_C$} \\
\midrule
GTCNN                   & $1\times 10^{-4}$ & 128 & 500 & Yes \\
GVARMA (P{=}1, Q{=}1, K{=}2) & $1\times 10^{-4}$ & 128 & 500 & Yes \\
GGRNN                   & $1\times 10^{-4}$ & 128 & 500 & Yes \\
GVNN                    & $1\times 10^{-4}$ & 128 & 500 & Yes \\
\bottomrule
\end{tabular*}
\end{table*}

\subsection{Experimental Details: Traffic Forecasting (METR--LA \& PEMS--BAY)}

\begin{figure*}[!htbp]
   \centering
    \includegraphics[width=0.9\linewidth]{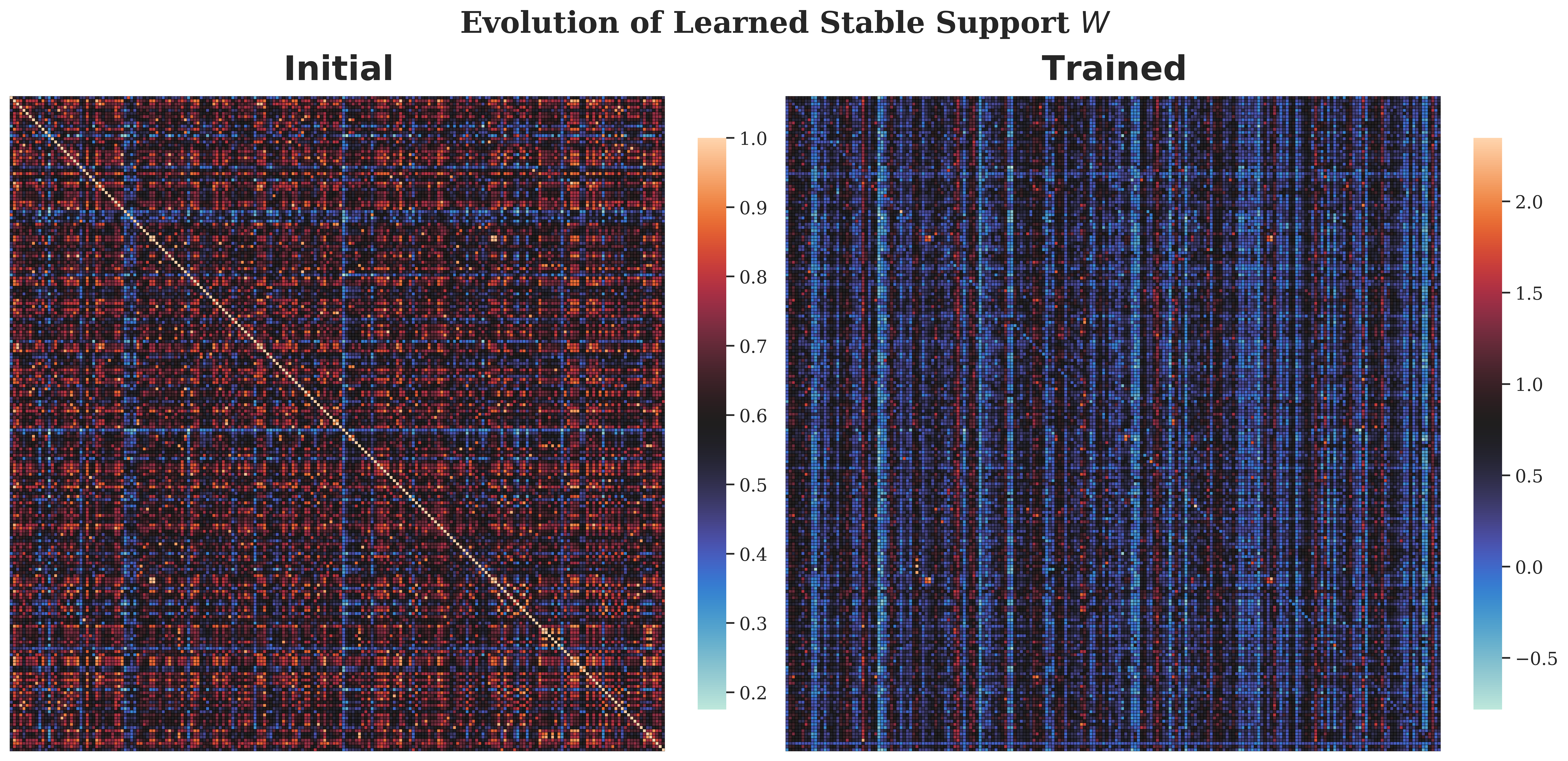}
    \caption{\textbf{Learned graph support matrix $W_C$ before and after training.} The figure illustrates how the static graph support matrix $W_C$ evolves through training. The left panel shows the initialized matrix, while the right panel presents the learned weights after optimization, revealing how the model adapts graph connectivity structure for improved forecasting.}
    \label{fig:learned_W_C}
\end{figure*}

We use a sliding window of \(T{=}6\) (30\,min) to forecast \(H\in\{3,6,12\}\) steps (15/30/60\,min). Data are split chronologically: \(80\%\) train\(+\)val and \(20\%\) test; within the first \(80\%\) we take \(80\%\) train and \(20\%\) validation. Inputs are z-scored \emph{per sample across channels}. Dynamic adjacencies are renormalized slice-wise as \(D^{-\frac12}(A{+}I)D^{-\frac12}\). We train with Adam and MSE loss for 200 epochs, select the best validation checkpoint, and evaluate on test. Runs use three seeds \(\{124,14,124235\}\). All models use 2 convolution layers with the GVNN having the LDE as the first layer and IC as second. The transformer and LSTM models also use 2 layers with the transformer only consisting of one attention head.

%---------------------------------------------------------
% Traffic datasets: dataset-level hyperparameters
%---------------------------------------------------------
\begin{table*}[!t]
\centering
\small
\caption{Dataset-level hyperparameters (both model families run on both datasets; the \emph{only} dataset difference is batch size).}
\label{tab:traffic-dataset-hparams}
\begin{tabular*}{\textwidth}{@{\extracolsep{\fill}}lccccccc@{}}
\toprule
\textbf{Dataset} & \textbf{T} & \textbf{H} & \textbf{Split} & \textbf{Batch} & \textbf{Epochs} & \textbf{Seeds} & \textbf{Optimizer / Loss} \\
\midrule
METR--LA  & 6 & 3, 6, 12 & 80/20 (chron.) & 1280 & 200 & 124, 14, 124235 & Adam / MSE \\
PEMS--BAY & 6 & 3, 6, 12 & 80/20 (chron.) & 1024 & 200 & 124, 14, 124235 & Adam / MSE \\
\bottomrule
\end{tabular*}
\end{table*}

\paragraph{Graph construction (used by all graph-based models).}
We build a static channel similarity matrix \(W_C\in\mathbb{R}^{C\times C}\) from channelwise Pearson correlations over the full training set . Models that mark \(W_C\) as \emph{trainable} initialize from this correlation and update it end-to-end; otherwise \(W_C\) is fixed. All dynamic operators \(\Omega(t)\) are fused with \(W_C\) by Hadamard product and renormalized slice-wise.

%---------------------------------------------------------
% Traffic models: hyperparameters and operator details
%---------------------------------------------------------
\begin{table*}[!t]
\centering
\small
\caption{Model hyperparameters and operator details (applied identically on METR--LA and PEMS--BAY).}
\label{tab:traffic-model-hparams}
\begin{tabular*}{\textwidth}{@{\extracolsep{\fill}}l l c c c c@{}}
\toprule
\textbf{Family} & \textbf{Model} & \textbf{LR} & \textbf{Hidden} & \textbf{Epochs} & \textbf{Trainable $W_C$} \\ 
\midrule

\multirow{5}{*}{Graph-based}
& GVNN (fixed $W_C$) & $1\times 10^{-4}$ & 128 & 200 & No \\
& GTCNN              & $1\times 10^{-4}$ & 128 & 200 & No \\
& GVARMA (P{=}1, Q{=}1, K{=}2) & $1\times 10^{-4}$ & 128 & 200 & No \\
& GGRNN              & $1\times 10^{-4}$ & 128 & 200 & No \\
& GVNN (trainable $W_C$) & $1\times 10^{-4}$ & 128 & 200 & Yes \\

\midrule
\multirow{2}{*}{Sequence-based}
& LSTM (2 layers)                  & $1\times 10^{-3}$ & 128 & 200 & --- \\
& Transformer (1 head, 2 layers)   & $1\times 10^{-3}$ & 128 & 200 & --- \\
\bottomrule
\end{tabular*}
\end{table*}

%========================================================
% EEG EXPERIMENTAL DETAILS (PhysioNet MI & BNCI2014_001)
% Requires in preamble: \usepackage{booktabs,makecell}
%========================================================

\footnote{For both EEG datasets the Transformer and LSTM models consisted of two layers while the GVNN was 1 layer.}

\paragraph{PhysioNet MI (binary: T1 vs.\ T2).}
We load raw EDF files from per-participant folders \texttt{S\{001..109\}} =, excluding faulty IDs
\(\{\texttt{088},\texttt{089},\texttt{092},\texttt{100}\}\).
For each valid subject we select only the motor imagery runs \(\texttt{R04},\texttt{R08},\texttt{R12}\), read EDF with \texttt{mne}, and extract events from annotations.
We dynamically map the annotation codes for \(\texttt{T1}\) and \(\texttt{T2}\), keep only those trials, and epoch each trial with
\(t_{\min}{=}0\) to \(t_{\max}{=}3.1\,\mathrm{s}\) at \(160\,\mathrm{Hz}\) (496 samples).
Trials and labels are concatenated across all participants.
We then perform stratified \(K{=}5\)-fold CV across \emph{all} trials (pooled cross-subject), building the stable support
\(W_C\) \emph{within each fold from training windows only} as the absolute channelwise Pearson correlation
\(\left|{\rm corr}\right|\) .
All models receive inputs normalized per sample across channels (z-score), and all graph operators use slice-wise symmetric
renormalization \(D^{-1/2}(A{+}I)D^{-1/2}\).

%---------------------------------------------------------
% PhysioNet MI: dataset protocol & hyperparameters
%---------------------------------------------------------
\begin{table*}[!t]
\centering
\small
\caption{PhysioNet MI: dataset-level protocol and hyperparameters.}
\label{tab:physionet-protocol}
\begin{tabular*}{\textwidth}{@{\extracolsep{\fill}}lcccccccc@{}}
\toprule
\textbf{Trials} & \textbf{Classes} & \textbf{Epoch} & \textbf{FS} & \textbf{CV} & \textbf{Batch} & \textbf{Epochs} & \textbf{LR / WD} & \textbf{Metrics} \\
\midrule
pooled (all subj.) & 2 (T1/T2) & 3.1\,s ($T{=}496$) & 160\,Hz & strat.\ 5-fold & 64 & 50 & $10^{-3}$ / $10^{-4}$ & Acc \\
\bottomrule
\end{tabular*}
\end{table*}

\paragraph{BNCI2014\_001 (BCI~2a, 4-class).}

We use MOABB/Braindecode \citep{MOABB2023} to load all subjects (1..9).
Preprocessing: pick EEG, scale by \(10^6\), band-pass \(0.01\!-\!20\,\mathrm{Hz}\), exponential moving standardization
(\( \texttt{factor\_new}{=}10^{-3},\ \texttt{init\_block\_size}{=}1000 \)).
Windows are created from events with a start offset of \(-0.5\,\mathrm{s}\) (MOABB defaults for stop/length are used).
We concatenate windows across subjects and run stratified 5-fold CV.
In each fold, \(W_C=\left|{\rm corr}\right|\) is computed from training windows only, and used by GVNN;
inputs are per-sample channel z-scored inside each model \citep{bci2a}.

%---------------------------------------------------------
% BNCI2014_001 (4-class): dataset protocol & hyperparameters
%---------------------------------------------------------
\begin{table*}[!t]
\centering
\small
\caption{BNCI2014\_001 (4-class): dataset-level protocol and hyperparameters.}
\label{tab:bci2a-protocol}
\begin{tabular*}{\textwidth}{@{\extracolsep{\fill}}lccccccc@{}}
\toprule
\textbf{Trials} & \textbf{Classes} & \textbf{Preproc} & \textbf{CV} & \textbf{Batch} & \textbf{Epochs} & \textbf{LR / WD} & \textbf{Metrics} \\
\midrule
pooled (all subj.) 
& 4 
& bp.\ $0.01{-}20$ Hz \& EMS 
& strat.\ 5-fold 
& 64 
& 100 
& $10^{-3}$ / $10^{-4}$ 
& Acc \\
\bottomrule
\end{tabular*}
\end{table*}

\subsection{Computational Complexity Analysis}
\label{subsec:complexity_analysis}

We compare two \emph{hypothetical} ways to realize signal-dependent graph convolution on inputs
$\mathbf{X}\in\mathbb{R}^{B\times C\times T}$, with a fixed spatial support
$\mathbf{W}_s\in\mathbb{R}^{C\times C}$ and temporal path adjacency $\mathbf{L}_T\in\mathbb{R}^{T\times T}$.

\begin{enumerate}[leftmargin=2.1em]
  \item \textbf{Naive Cartesian (Kronecker) Method.}
  For each sample $b$, compute per-time masked connectivity and then build a full spatiotemporal kernel
  by the Kronecker product with $\mathbf{L}_T$, yielding $\mathbf{K}_b\in\mathbb{R}^{(CT)\times(CT)}$, and apply $\mathbf{K}_b$ to $\hat{\mathbf{x}}_b$.
  
  \item \textbf{Proposed Graph-Variate Low-Rank Batched Method.}
  Construct rank-1 (\texttt{IC}) or \emph{rank-3 expanded quadratic} (\texttt{LDE}) connectivities on-the-fly,
  mask by $\mathbf{W}_s$ via Hadamard product, and perform $T$ batched mat--vecs \emph{without} explicit Kronecker expansion.
\end{enumerate}

\subsubsection*{1.\ Naive Product Graph Method}
For each $b=1,\dots,B$ and $t=1,\dots,T$:
\begin{enumerate}[label=(\alph*),leftmargin=1.8em]
  \item \emph{Per-time connectivity \& masking:}
  \begin{align}
    \mathbf{S}_{b,t}^{\text{IC}} &= \mathbf{x}_{b,:,t}\,\mathbf{x}_{b,:,t}^{\top}, \\
    \mathbf{S}_{b,t}^{\text{LDE}} &= (\mathbf{x}_{b,:,t}\!\odot \mathbf{x}_{b,:,t})\,\mathbf{1}^{\top} \notag \\
                         &\quad + \mathbf{1}\,(\mathbf{x}_{b,:,t}\!\odot \mathbf{x}_{b,:,t})^{\top} \notag \\
                         &\quad - 2\,\mathbf{x}_{b,:,t}\,\mathbf{x}_{b,:,t}^{\top}, \\
    \widetilde{\mathbf{S}}_{b,t} &= \mathbf{S}_{b,t}\circ \mathbf{W}_s.
  \end{align}
  
  \item \emph{Stacking:}
  \begin{equation}
    \widetilde{\mathbf{S}}_b = [\widetilde{\mathbf{S}}_{b,1},\ldots,\widetilde{\mathbf{S}}_{b,T}] \in \mathbb{R}^{C\times C\times T}.
  \end{equation}
  
  \item \emph{Kronecker expansion \& apply:}
  \begin{align}
    \mathbf{K}_b &= \mathbf{L}_T \otimes \widetilde{\mathbf{S}}_b \in \mathbb{R}^{(CT)\times(CT)}, \\
    \hat{\mathbf{y}}_b &= \mathbf{K}_b\,\hat{\mathbf{x}}_b.
  \end{align}
\end{enumerate}

\paragraph{Complexity.}
Per-time connectivity: $O(BC^2T)$; kernel formation: $O(BC^2T^2)$; application: $O(BC^2T^2)$;
memory for all $\mathbf{K}_b$: $O(BC^2T^2)$. Net time: $O(BC^2T^2)$; memory: $O(BC^2T^2)$.

\subsubsection*{2.\ Proposed Graph-Variate Low-Rank Batched Method}
Form, for all $(b,t)$ in parallel:
\begin{align}
  \mathbf{J}_{b,t}^{\text{IC}} &= \mathbf{x}_{b,:,t}\,\mathbf{x}_{b,:,t}^{\top}, \\
  \mathbf{J}_{b,t}^{\text{LDE}} &= (\mathbf{x}_{b,:,t}\!\odot \mathbf{x}_{b,:,t})\,\mathbf{1}^{\top} \notag \\
                       &\quad + \mathbf{1}\,(\mathbf{x}_{b,:,t}\!\odot \mathbf{x}_{b,:,t})^{\top} \notag \\
                       &\quad - 2\,\mathbf{x}_{b,:,t}\,\mathbf{x}_{b,:,t}^{\top},
\end{align}
then mask with $\mathbf{W}_s$: $\boldsymbol{\Omega}_{b,t}=\mathbf{J}_{b,t}\circ \mathbf{W}_s$ (using the appropriate case).
All $T$ masked matrices live implicitly inside $\boldsymbol{\Omega}\in\mathbb{R}^{B\times C\times C\times T}$.
We then compute, in one batched call,
\begin{equation}
  \mathbf{y}_{b,:,t}=\boldsymbol{\Omega}_{b,t}\,\mathbf{x}_{b,:,t},
\end{equation}
vectorizing over $b$ and $t$.

\paragraph{Complexity.}
Connectivity and masking: $O(BC^2T)$; $T$ batched mat--vecs: $O(BC^2T)$; memory: $O(BC^2T)$.
Net time: $O(BC^2T)$; memory: $O(BC^2T)$.

\begin{table}[t]
\centering
\small
\caption{Asymptotic comparison: naive Cartesian vs.\ batched low-rank.}
\label{tab:complexity_summary}
\begin{tabular}{@{}lcc@{}}
\toprule
\textbf{Method} & \textbf{Time} & \textbf{Memory} \\
\midrule
Naive Cartesian & $O(BC^2T^2)$ & $O(BC^2T^2)$ \\
Batched Low-Rank (ours) & $O(BC^2T)$ & $O(BC^2T)$ \\
\bottomrule
\end{tabular}
\end{table}

\begin{tcolorbox}[colback=blue!3!white,colframe=blue!40!black,boxrule=0.4pt,width=\columnwidth]
\textbf{Takeaway.}
Avoiding explicit Kronecker formation with $\mathbf{L}_T$ removes the quadratic dependence on $T$ in both
compute and memory. Using rank-1 (\texttt{IC}) and rank-3 expanded quadratic (\texttt{LDE}) constructions,
plus Hadamard masking and batched mat--vecs, yields linear $O(BC^2T)$ execution. 
\end{tcolorbox}
\subsubsection*{Main Convolution}

\begin{lstlisting}[style=pytiny,basicstyle=\ttfamily\scriptsize,breaklines=true,caption={Core PyTorch implementation.}]
import torch

EPS = 1e-5

def renormalize_dynamic(A, eps=EPS):
    """
    A: (B, C, C, T) dynamic affinity
    Returns symmetric renorm: D^{-1/2}(A+I)D^{-1/2}
    """
    I = torch.eye(A.size(1), device=A.device)
    I = I[None, :, :, None]  # (1, C, C, 1)
    At = A + I
    deg = At.sum(2, keepdim=True)  # (B, C, 1, T)
    inv = deg.clamp(min=eps).pow(-0.5)
    S = inv * At * inv.transpose(1, 2)
    return S

def graph_variate(x, fun='corr', Zave=True, eps=EPS, W):
    """
    x: (B, C, T)
    returns normalized dynamic adjacency Om: (B,C,C,T)
    """
    B, C_, T_ = x.shape
    if Zave:
        mu = x.mean(1, keepdim=True)
        sig = x.std(1, keepdim=True, unbiased=True)
        x = (x - mu) / (sig + eps)

    if fun == 'sqd':
        Om = (x.unsqueeze(2) - x.unsqueeze(1)).pow(2)
    elif fun == 'corr':
        D = x - x.mean(2, keepdim=True)
        Om = D.unsqueeze(2) * D.unsqueeze(1)

    return W*Om

def graph_conv(x, Om):
    """
    x:  (B, C, T)
    Om: (B, C, C, T) dynamic adjacency
    returns: (B, C, T)
    """
    Om_t = renormalize_dynamic(Om.permute(0, 3, 1, 2))   # (B, T, C, C)
    sig_t = x.permute(0, 2, 1).unsqueeze(-1)
    out = torch.matmul(Om_t, sig_t).squeeze(-1)
    return out.permute(0, 2, 1)
\end{lstlisting}

\noindent In practice, we build $\Omega$ via \texttt{graph\_variate}, apply the mask (Hadamard with $W_s$), optionally call \texttt{renormalize\_dynamic} slice-wise, and then use \texttt{graph\_conv} to perform all $BT$ mat–vecs in one call—achieving $O(BC^2T)$ time and memory.

\subsection{Rank-Lifting of the LDE Connectivity Profile}
\label{rankliftlde}

\begin{figure*}[t]
  \centering
  \includegraphics[width=\textwidth]{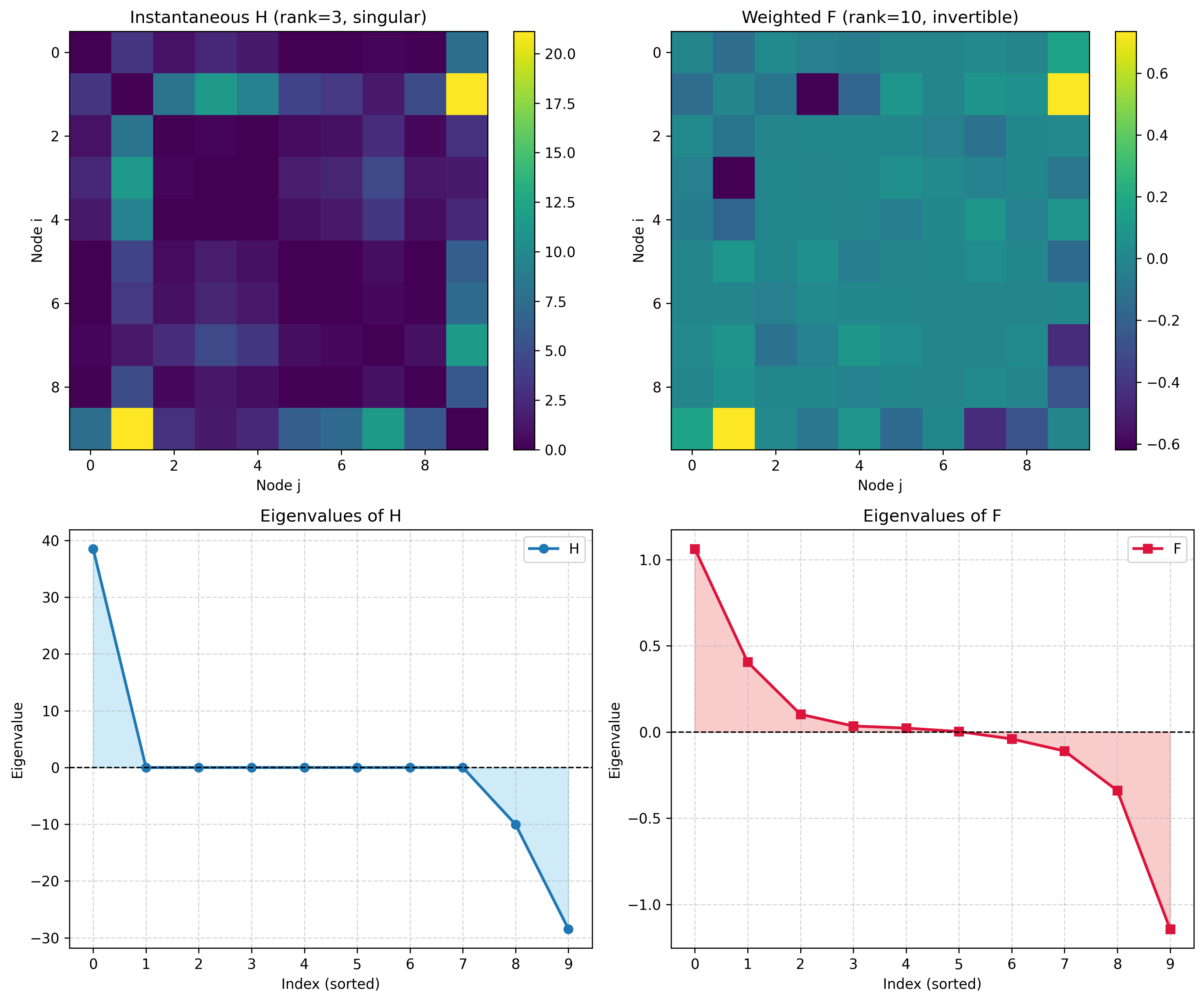}
  \caption{
    \textbf{Top-left:} The instantaneous squared-difference matrix $\mathbf{J}$ at a single time point, which is rank-deficient ($\mathrm{rank}=3$) and singular.
    \textbf{Top-right:} The Hadamard-weighted matrix $\boldsymbol{\Omega} = \mathbf{W}\circ \mathbf{J}$, where $\mathbf{W}$ is the long-term correlation; weighting lifts $\mathbf{J}$ to full rank ($\mathrm{rank}=10$) and makes $\boldsymbol{\Omega}$ invertible.
    \textbf{Bottom-left:} Sorted eigenvalues of $\mathbf{J}$, displaying exactly three nonzero modes and seven zeros.
    \textbf{Bottom-right:} Sorted eigenvalues of $\boldsymbol{\Omega}$, all ten nonzero and of mixed sign, confirming that $\boldsymbol{\Omega}$ is indefinite but invertible.
  }
  \label{fig:connectivity_summary}
\end{figure*}

\begin{theorem}
\label{thm:lde_rank_lifting}
Let $x_1,\dots,x_N$ be $N$ distinct real numbers and define the instantaneous squared-difference matrix
\[
\mathbf{J}(t)\in\mathbb{R}^{N\times N},
\quad
J_{ij}(t)=(x_i - x_j)^2,
\quad
J_{ii}(t)=0.
\]
Let 
\[
\mathcal{W} = \bigl\{\,\mathbf{W}\in\mathbb{R}^{N\times N}:w_{ij}\neq0 \text{ for all } i\neq j\bigr\},
\]
and for each $\mathbf{W}\in\mathcal{W}$ form the Hadamard product 
\[
\boldsymbol{\Omega}(t) = \mathbf{J}(t)\circ \mathbf{W},
\qquad 
\Omega_{ij}=J_{ij}(t)\,w_{ij}.
\]
Then:
\begin{enumerate}
  \item $\mathrm{rank}(\mathbf{J})\le3$, hence $\det(\mathbf{J}(t))=0$ and $\mathbf{J}$ is singular.
  \item The determinant
    \[
      P(w_{12},w_{13},\dots,w_{N-1,N})
      =\det\bigl(\mathbf{J}\circ \mathbf{W}\bigr)
    \]
    is a nonzero polynomial in the off-diagonal entries of $\mathbf{W}$. Consequently, outside its algebraic zero-locus of Lebesgue measure~0,
    \[
      \det(\mathbf{J}(t)\circ \mathbf{W})\neq0,
      \quad
      \mathrm{rank}(\mathbf{J}(t)\circ \mathbf{W})=N,
    \]
    so the Hadamard-weighted matrix is generically invertible.
\end{enumerate}
\end{theorem}

\begin{proof}
\textbf{(1) $\mathrm{rank}(\mathbf{J}(t))\le3$.} Define column-vectors in $\mathbb{R}^N$ by
\[
u_i = x_i^2(t),
\quad
v_i = x_i(t),
\quad
\mathbf{1}_i = 1.
\]
Then
\[
J_{ij}(t) = (x_i(t) - x_j(t))^2
       = u_i\,\mathbf{1}_j - 2\,v_i\,v_j + \mathbf{1}_i\,u_j,
\]
so in matrix form
\[
\mathbf{J}(t)= \mathbf{u}\,\mathbf{1}^\top - 2\,\mathbf{v}\,\mathbf{v}^\top + \mathbf{1}\,\mathbf{u}^\top.
\]
Each term on the right is rank~1, hence
$\mathrm{rank}(\mathbf{J}(t))\le1+1+1=3$. In particular when $N>3$, $\mathbf{J}(t)$ is singular and $\det(\mathbf{J}(t))=0$.

\medskip
\textbf{(2) $\det(\mathbf{J}(t)\circ \mathbf{W})$ is a nonzero polynomial.}  
By the Leibniz formula,
\begin{align*}
\det(\boldsymbol{\Omega}(t))
&= \sum_{\pi \in S_N} \mathrm{sgn}(\pi)\, \prod_{i=1}^N \Omega_{i,\pi(i)}(t) \\
&= \sum_{\substack{\pi \in S_N \\ \pi(i) \ne i\; \forall i}}
    \mathrm{sgn}(\pi)\, \prod_{i=1}^N \bigl[J_{i,\pi(i)}(t)\, w_{i,\pi(i)}\bigr],
\end{align*}
since $\Omega_{ii}(t)=0$ kills any term with a fixed point. Thus
\begin{align}
&\det(\mathbf{J}(t)\circ \mathbf{W}) \notag \\
&= \sum_{\substack{\pi\in S_N\\\pi(i)\neq i}}
\Bigl(\mathrm{sgn}(\pi)\prod_{i=1}^N J_{i,\pi(i)}(t)\Bigr)
\Bigl(\prod_{i=1}^N w_{i,\pi(i)}\Bigr),
\end{align}
a multivariate polynomial $P(\{w_{ij}\})$ in the off-diagonal $w_{ij}$.

To show $P\not\equiv0$, pick the $N$-cycle
$\pi_0\colon i\mapsto i+1\pmod N$. Its monomial is
\[
\prod_{i=1}^N w_{i,\pi_0(i)}
= w_{1,2}\,w_{2,3}\cdots w_{N-1,N}\,w_{N,1},
\]
and its coefficient is
\begin{align}
&\mathrm{sgn}(\pi_0)\,\prod_{i=1}^N J_{i,\pi_0(i)}(t) \notag \\
&\quad =\pm\,(x_1-x_2)^2(x_2-x_3)^2\cdots(x_N-x_1)^2 \neq 0
\end{align}
because the $x_i$ are distinct. Hence $P$ has at least one nonzero coefficient and so is not the zero polynomial. Therefore it vanishes only on a proper hypersurface in~$\mathcal{W}$, proving that for almost every full-support $\mathbf{W}$,  
$\det(\mathbf{J}(t)\circ \mathbf{W})\neq0$ and $\mathrm{rank}(\mathbf{J}(t)\circ \mathbf{W})=N$.
\end{proof}
\subsection{Stability of GVNN Layer}

\begin{theorem}[GVNN Layer is Globally Lipschitz]
\label{thm:gvnn_lipschitz}
Let $\mathbf{W}\in\mathbb{R}^{N\times N}$ be symmetric with nonnegative entries, and define
\[
\alpha = \max_{1\le i\le N}\sum_{j=1}^N w_{ij}.
\]
Let 
\[
\mathbf{X} = [\,\mathbf{x}(1)\;\dots\;\mathbf{x}(T)\,] \in \mathbb{R}^{N\times T},
\]
and write
\begin{align*}
\mu_i 
&= \frac{1}{T} \sum_{t=1}^T x_i(t), \\[4pt]
M 
&= \max_{\substack{1 \le i \le N \\ 1 \le t \le T}} | x_i(t) - \mu_i |, \\[4pt]
B 
&= \max_{\substack{1 \le i \le N \\ 1 \le t \le T}} | x_i(t) |.
\end{align*}

Let scalar filters $\mathbf{a}=(a_t)_{t=1}^T$ and $\mathbf{b}=(b_t)_{t=1}^T$ satisfy
\[
a^{\star}=\max_{1\le t\le T}|a_t|,
\quad
b^{\star}=\max_{1\le t\le T}|b_t|.
\]
For each $t$ define two node functions:
\begin{align}
J^{\mathrm{IC}}_{ij}(t) &= |(x_i(t)-\mu_i)(x_j(t)-\mu_j)|, \\
J^{\mathrm{LDE}}_{ij}(t) &= (x_i(t)-x_j(t))^2,
\end{align}
and form the Hadamard products
\begin{align}
\boldsymbol{\Omega}^{\mathrm{IC}}(t) &= \mathbf{W}\circ \mathbf{J}^{\mathrm{IC}}(t), \\
\boldsymbol{\Omega}^{\mathrm{LDE}}(t) &= \mathbf{W}\circ \mathbf{J}^{\mathrm{LDE}}(t).
\end{align}
Given any pointwise-1-Lipschitz nonlinearity $\sigma:\mathbb{R}\to\mathbb{R}$, define
\begin{align}
\mathbf{y}(t) &= \sigma\bigl(a_t\,\mathbf{x}(t) + b_t\,\boldsymbol{\Omega}(t)\,\mathbf{x}(t)\bigr), \\
F(\mathbf{X}) &= [\,\mathbf{y}(1)\;\dots\;\mathbf{y}(T)] \in \mathbb{R}^{N\times T}.
\end{align}
Then for every pair $\mathbf{X},\mathbf{X}'\in\mathbb{R}^{N\times T}$:
\begin{align}
&\|F(\mathbf{X})-F(\mathbf{X}')\|_{F} \notag \\
&\quad \le
(a^{\star} + \alpha b^{\star} M^2)\|\mathbf{X}-\mathbf{X}'\|_{F}
\;\;(\mathrm{IC}), \\
&\|F(\mathbf{X})-F(\mathbf{X}')\|_{F} \notag \\
&\quad \le
(a^{\star} + 4\alpha b^{\star} B^2)\|\mathbf{X}-\mathbf{X}'\|_{F}
\;\;(\mathrm{LDE}).
\end{align}
\end{theorem}

\begin{proof}
Because $\mathbf{W}$ is symmetric with nonnegative entries, Gershgorin's circle theorem guarantees that every eigenvalue $\lambda$ of $\mathbf{W}$ lies in the interval $[0,\alpha]$. Consequently, the spectral (operator) norm of $\mathbf{W}$ satisfies
\[
\|\mathbf{W}\|_{\mathrm{op}} \le \alpha.
\]

Fix an arbitrary time index $t\in\{1,\dots,T\}$. We treat the IC and LDE cases in parallel, noting only where the node function definition differs.

\paragraph{IC case.}
Define the diagonal matrix
\[
\mathbf{D}(t) = \mathrm{diag}\bigl(|x_i(t)-\mu_i|\bigr)_{i=1}^N.
\]
By definition of $M$,
\[
\|\mathbf{D}(t)\|_{\mathrm{op}}
= \max_{1\le i\le N}|x_i(t)-\mu_i|
\le M.
\]
Since the Hadamard product with $\mathbf{J}^{\mathrm{IC}}(t)$ coincides with the congruence
\[
\boldsymbol{\Omega}^{\mathrm{IC}}(t) = \mathbf{D}(t)\,\mathbf{W}\,\mathbf{D}(t),
\]
submultiplicativity of the operator norm yields
\begin{align*}
&\|\boldsymbol{\Omega}^{\mathrm{IC}}(t)\|_{\mathrm{op}} \\
&\quad \le \|\mathbf{D}(t)\|_{\mathrm{op}} \cdot \|\mathbf{W}\|_{\mathrm{op}} \cdot \|\mathbf{D}(t)\|_{\mathrm{op}} \\
&\quad \le M \cdot \alpha \cdot M = \alpha M^2.
\end{align*}

\paragraph{LDE case.}
Here each entry of the instantaneous matrix is $(x_i(t)-x_j(t))^2$. We bound this directly in terms of the maximum node value $B$:
\begin{align*}
(x_i(t) - x_j(t))^2 
&= |x_i(t) - x_j(t)|^2 \\
&\le (|x_i(t)| + |x_j(t)|)^2 \\
&\le (B + B)^2 = 4B^2.
\end{align*}

Therefore, for every $i, j$,
\[
|\Omega^{\mathrm{LDE}}_{ij}(t)|
= w_{ij}(x_i(t) - x_j(t))^2
\le 4B^2 w_{ij}.
\]

Summing over $j$ shows that each row sum of $|\boldsymbol{\Omega}^{\mathrm{LDE}}(t)|$ is at most
\[
\sum_{j}4B^2 w_{ij} = 4B^2\sum_j w_{ij} \le 4\alpha B^2.
\]
Since $\boldsymbol{\Omega}^{\mathrm{LDE}}(t)$ remains symmetric with nonnegative entries, its operator norm is upper bounded by its maximum row sum:
\begin{align*}
&\|\boldsymbol{\Omega}^{\mathrm{LDE}}(t)\|_{\mathrm{op}}
= \rho\bigl(\boldsymbol{\Omega}^{\mathrm{LDE}}(t)\bigr) \\
&\quad \le \max_{1\le i\le N}\sum_{j=1}^N \Omega^{\mathrm{LDE}}_{ij}(t)
\le 4\alpha B^2.
\end{align*}

\paragraph{Lipschitz bound.}
In either case define the map $g_t:\mathbb{R}^N\to\mathbb{R}^N$ by
\[
g_t(\mathbf{z}) = a_t\,\mathbf{z} + b_t\,\boldsymbol{\Omega}(t)\,\mathbf{z}.
\]
For any two vectors $\mathbf{u},\mathbf{v}\in\mathbb{R}^N$, we have
\[
g_t(\mathbf{u})-g_t(\mathbf{v})
= \bigl(a_t\mathbf{I} + b_t\boldsymbol{\Omega}(t)\bigr)(\mathbf{u}-\mathbf{v}).
\]
Applying the triangle inequality together with the operator-norm bound on $\boldsymbol{\Omega}(t)$ yields
\begin{align*}
&\|g_t(\mathbf{u})-g_t(\mathbf{v})\|_2 \\
&\quad \le |a_t|\|\mathbf{u}-\mathbf{v}\|_2
 + |b_t|\|\boldsymbol{\Omega}(t)\|_{\mathrm{op}}\|\mathbf{u}-\mathbf{v}\|_2.
\end{align*}
Since $|a_t|\le a^\star$ and $|b_t|\le b^\star$, it follows that
\begin{align*}
&\|g_t(\mathbf{u})-g_t(\mathbf{v})\|_2 \\
&\quad \le
(a^\star + b^\star\|\boldsymbol{\Omega}(t)\|_{\mathrm{op}})\|\mathbf{u}-\mathbf{v}\|_2.
\end{align*}

\paragraph{Combining with nonlinearity.}
Because $\sigma$ is pointwise 1-Lipschitz, for each $t$ and each pair of signals $\mathbf{x}(t), \mathbf{x}'(t)$,
\begin{align*}
&\|\mathbf{y}(t) - \mathbf{y}'(t)\|_2 \\
&\quad = \|\sigma(g_t(\mathbf{x}(t))) - \sigma(g_t(\mathbf{x}'(t)))\|_2 \\
&\quad \le \|g_t(\mathbf{x}(t)) - g_t(\mathbf{x}'(t))\|_2 \\
&\quad \le L\|\mathbf{x}(t) - \mathbf{x}'(t)\|_2,
\end{align*}
where
\[
L =
\begin{cases}
a^\star + \alpha b^{\star} M^2, & \text{IC}, \\[4pt]
a^\star + 4\alpha b^{\star} B^2, & \text{LDE}.
\end{cases}
\]

Finally, summing these squared-norm inequalities over $t = 1, \dots, T$ and taking the square root gives:
\begin{align*}
&\|F(\mathbf{X}) - F(\mathbf{X}')\|_F \\
&\quad = \left(\sum_{t=1}^T \|\mathbf{y}(t) - \mathbf{y}'(t)\|_2^2\right)^{1/2} \\
&\quad \le L \left(\sum_{t=1}^T \|\mathbf{x}(t) - \mathbf{x}'(t)\|_2^2\right)^{1/2} \\
&\quad = L\|\mathbf{X} - \mathbf{X}'\|_F.
\end{align*}
This completes the proof.
\end{proof}
\subsection{Extended Theorems and Proofs}

This section includes additional theorems and propositions for completeness.

\begin{theorem}[IC Spectral bounds under amplitude-scaling]
\label{thm:amplitude_scaling_bounds}
Let $\mathbf{W} \in \mathbb{S}_{++}^N$ have spectrum
$\lambda_{\min}(\mathbf{W})\le \dots \le \lambda_{\max}(\mathbf{W})$,
and at time $t$ let the centred sample $\tilde{\mathbf{x}}_t\in\mathbb{R}^N$ satisfy 
$\tilde{x}_{i}(t)\neq0$ for all $i$. Define
\begin{align}
  \mathbf{D}_t &= \diag\!\bigl(|\tilde{\mathbf{x}}_t|\bigr), \quad
  \boldsymbol{\rho}_t = \mathbf{D}_t\,\mathbf{W}\,\mathbf{D}_t, \\
  m_t &= \min_i|\tilde{x}_i(t)|, \quad
  M_t = \max_i|\tilde{x}_i(t)|.
\end{align}
If $\delta_{1,t}\le\cdots\le\delta_{N,t}$ are the eigenvalues of $\boldsymbol{\rho}_t$,
then for each $i=1,\dots,N$,
\[
  m_t^2\,\lambda_{\min}(\mathbf{W})
  \;\le\;
  \delta_{i,t}
  \;\le\;
  M_t^2\,\lambda_{\max}(\mathbf{W}).
\]
\end{theorem}

\begin{proof}
Recall the Rayleigh quotient of a symmetric matrix $\mathbf{A}$ and nonzero $\mathbf{w}$ is
\[
  \mathcal{R}(\mathbf{A};\mathbf{w})
  := \frac{\mathbf{w}^\top \mathbf{A}\,\mathbf{w}}{\mathbf{w}^\top \mathbf{w}}.
\]
By the Rayleigh--Ritz theorem (a special case of the Courant--Fischer min--max theorem),
\[
  \lambda_{\min}(\mathbf{A})
  \le
  \mathcal{R}(\mathbf{A};\mathbf{w})
  \le
  \lambda_{\max}(\mathbf{A}),
  \quad
  \forall\,\mathbf{w}\neq\mathbf{0},
\]
and the eigenvalues of $\mathbf{A}$ coincide with the extremal values of $\mathcal{R}(\mathbf{A};\mathbf{w})$ over appropriate subspaces.

For any unit vector $\mathbf{v}\in\mathbb{R}^N$ ($\|\mathbf{v}\|=1$), consider
\[
  \mathbf{v}^\top \boldsymbol{\rho}_t\,\mathbf{v}
  = \mathbf{v}^\top (\mathbf{D}_t\,\mathbf{W}\,\mathbf{D}_t)\,\mathbf{v}
  = (\mathbf{D}_t\,\mathbf{v})^\top \mathbf{W}\,(\mathbf{D}_t\,\mathbf{v}).
\]
We bound $(\mathbf{D}_t \mathbf{v})^\top \mathbf{W}\,(\mathbf{D}_t \mathbf{v})$ using $\mathcal{R}(\mathbf{W};\cdot)$.

Define $\mathbf{u} = \mathbf{D}_t \mathbf{v} / \|\mathbf{D}_t \mathbf{v}\|$, so $\mathbf{u}\neq\mathbf{0}$ and $\|\mathbf{u}\|=1$. Then
\[
  (\mathbf{D}_t\,\mathbf{v})^\top \mathbf{W}\,(\mathbf{D}_t\,\mathbf{v})
  = \|\mathbf{D}_t\,\mathbf{v}\|^2 \,\mathcal{R}(\mathbf{W};\mathbf{u}).
\]
By the Rayleigh--Ritz result,
\[
  \lambda_{\min}(\mathbf{W})\le\mathcal{R}(\mathbf{W};\mathbf{u})\le\lambda_{\max}(\mathbf{W}),
\]
so
\begin{align}
  (\mathbf{D}_t \mathbf{v})^\top \mathbf{W}\,(\mathbf{D}_t \mathbf{v})
  &\in \bigl[\lambda_{\min}(\mathbf{W})\,\|\mathbf{D}_t \mathbf{v}\|^2, \notag \\
  &\qquad \lambda_{\max}(\mathbf{W})\,\|\mathbf{D}_t \mathbf{v}\|^2\bigr].
\end{align}

Since $\mathbf{v}$ has $\|\mathbf{v}\|=1$ and $\mathbf{D}_t = \diag(d_{1,t},\dots,d_{N,t})$ with $d_{i,t}=|\tilde{x}_i(t)|\in [m_t,M_t]$, we have
\[
  \|\mathbf{D}_t \mathbf{v}\|^2
  = \sum_{i=1}^N d_{i,t}^2\,v_i^2
  \in [\,m_t^2,\;M_t^2\,].
\]
Therefore for every unit $\mathbf{v}$,
\begin{align}
  \mathbf{v}^\top \boldsymbol{\rho}_t\,\mathbf{v}
  &= (\mathbf{D}_t \mathbf{v})^\top \mathbf{W}\,(\mathbf{D}_t \mathbf{v}) \notag \\
  &\in \bigl[m_t^2\,\lambda_{\min}(\mathbf{W}),\;M_t^2\,\lambda_{\max}(\mathbf{W})\bigr].
\end{align}

Finally, by the Courant--Fischer characterization of eigenvalues, the $i$th largest eigenvalue $\delta_{i,t}$ of $\boldsymbol{\rho}_t$ is the extremal Rayleigh quotient over an $i$-dimensional subspace. Since \emph{all} Rayleigh quotients lie in $[m_t^2\lambda_{\min}(\mathbf{W}),\,M_t^2\lambda_{\max}(\mathbf{W})]$, each $\delta_{i,t}$ must also satisfy
\[
  m_t^2\,\lambda_{\min}(\mathbf{W})
  \le
  \delta_{i,t}
  \le
  M_t^2\,\lambda_{\max}(\mathbf{W}),
\]
for $i=1,\dots,N$. This completes the proof.
\end{proof}

\begin{theorem}[IC Condition-number bound under amplitude-scaling]
\label{thm:condition_number_bound}
Under the hypotheses of Theorem~\ref{thm:schur_full_rank}, let
\[
  d_{i} = |\tilde{x}_i^{(m)}(t)|,
  \quad
  d_{\min} = \min_{1\le i\le N}d_{i},
  \quad
  d_{\max} = \max_{1\le i\le N}d_{i},
\]
and recall $\mathbf{W}\in\mathbb{S}_{++}^N$ has spectrum
$\lambda_{\min}(\mathbf{W})\le\dots\le\lambda_{\max}(\mathbf{W})$. Then the instantaneous filtered matrix $\boldsymbol{\rho}_t=\mathbf{D}_t\,\mathbf{W}\,\mathbf{D}_t$ is SPD and its condition number satisfies
\[
  \kappa\bigl(\boldsymbol{\rho}_t\bigr)
  =
  \frac{\lambda_{\max}\bigl(\boldsymbol{\rho}_t\bigr)}
       {\lambda_{\min}\bigl(\boldsymbol{\rho}_t\bigr)}
  \le
  \frac{d_{\max}^2}{d_{\min}^2}
  \cdot
  \frac{\lambda_{\max}(\mathbf{W})}{\lambda_{\min}(\mathbf{W})}.
\]
\end{theorem}

\begin{proof}
From Theorem~\ref{thm:schur_full_rank}, $\boldsymbol{\rho}_t^{(m)}$ is congruent to the SPD matrix $\mathbf{W}$, so it remains SPD, hence all eigenvalues are strictly positive and the condition number is well-defined.

From the Rayleigh--Ritz characterization, for any unit vector $\mathbf{v}$,
\begin{align*}
\mathbf{v}^\top \boldsymbol{\rho}_t \mathbf{v}
&= (\mathbf{D}_t\mathbf{v})^\top\, \mathbf{W}\, (\mathbf{D}_t\mathbf{v}) \\
&\in \bigl[\lambda_{\min}(\mathbf{W})\, \|\mathbf{D}_t \mathbf{v}\|^2,\;
           \lambda_{\max}(\mathbf{W})\, \|\mathbf{D}_t \mathbf{v}\|^2\bigr].
\end{align*}

Since $d_{\min} \le d_i \le d_{\max}$ for all $i$, and $\|\mathbf{v}\| = 1$, one checks
\[
  d_{\min}^2 \le\|\mathbf{D}_t\mathbf{v}\|^2\le d_{\max}^2.
\]
Hence every eigenvalue $\delta$ of $\boldsymbol{\rho}_t^{(m)}$ satisfies
\[
  d_{\min}^2\,\lambda_{\min}(\mathbf{W})
  \le
  \delta
  \le
  d_{\max}^2\,\lambda_{\max}(\mathbf{W}).
\]

Writing $\delta_{\min}=\lambda_{\min}(\boldsymbol{\rho}_t)$ and $\delta_{\max}=\lambda_{\max}(\boldsymbol{\rho}_t)$, the above yields
\[
  \delta_{\min}\ge d_{\min}^2\,\lambda_{\min}(\mathbf{W}),
  \quad
  \delta_{\max}\le d_{\max}^2\,\lambda_{\max}(\mathbf{W}).
\]
Therefore
\[
  \kappa\bigl(\boldsymbol{\rho}_t\bigr)
  = \frac{\delta_{\max}}{\delta_{\min}}
  \le
  \frac{d_{\max}^2\,\lambda_{\max}(\mathbf{W})}
       {d_{\min}^2\,\lambda_{\min}(\mathbf{W})},
\]
which completes the proof.
\end{proof}

\label{ssec:gershgorin_bounds}

\begin{figure*}[t]
  \centering
  \includegraphics[width=0.8\textwidth]{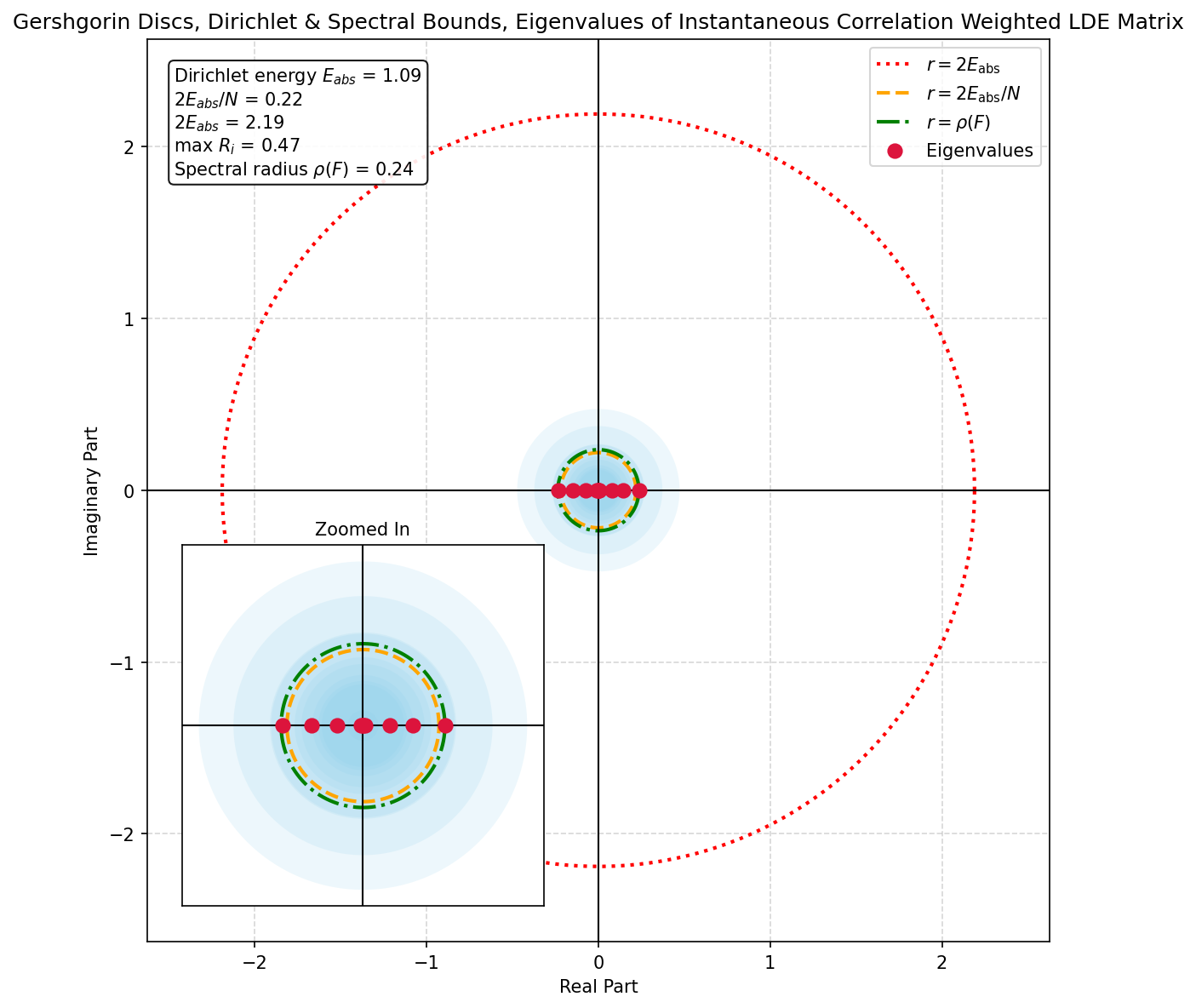}
  \caption{
    \textbf{Main panel:} Light-blue shaded circles show the Gershgorin discs of the Hadamard-weighted matrix $\boldsymbol{\Omega} = \mathbf{W}\circ \mathbf{J}$, each centered at the origin with radius 
    $R_i = \sum_{j\neq i}|w_{ij}|(x_i-x_j)^2$. Red dotted circle marks the upper Dirichlet-energy bound $r=2\mathcal{E}_{\mathrm{abs}}$, orange dashed circle marks the average-energy bound $r=2\mathcal{E}_{\mathrm{abs}}/N$, and green dash--dot circle marks the spectral radius $r=\rho(\boldsymbol{\Omega})$. Red crosses are the eigenvalues of $\boldsymbol{\Omega}$, all lying within the union of the Gershgorin discs.
    \textbf{Inset:} A close-up around the origin shows the small Gershgorin discs, the tight Dirichlet lower bound $2\mathcal{E}_{\mathrm{abs}}/N$, and the spectral-radius circle relative to the cluster of eigenvalues.
  }
  \label{fig:gershgorin_dirichlet_inset}
\end{figure*}
\begin{theorem}[Gershgorin bounds on $\boldsymbol{\rho}_t$]  
\label{thm:gershgorin_instantaneous}
Let $\boldsymbol{\rho}_t\in\mathbb{S}_{++}^N$ as above, and define
\[
  a_{ii} = \rho_{ii} = w_{ii} d_{i,t}^2,
  \qquad
  R_i = \sum_{j\ne i} |w_{ij}| d_{i} d_{j}.
\]
Then every eigenvalue $\delta_i$ of $\boldsymbol{\rho}_t$ satisfies
\[
  \delta_i \in
  \bigcup_{i=1}^N D(a_{ii}, R_i).
\]
In particular, since $\mathbf{W}$ is SPD and its diagonal entries $w_{ii}>0$, each disc lies strictly in the right-half plane and hence $\boldsymbol{\rho}_t$ has all positive eigenvalues.

Moreover, letting
\begin{align*}
  d_{\min} &= \min_i d_{i}, \\
  d_{\max} &= \max_i d_{i}, \\
  r_{\max} &= \max_i\sum_{j\ne i}|w_{ij}|,
\end{align*}
we obtain the simplified bound
\[
  \delta_i \in
  \bigl[d_{\min}^2 w_{\min} - d_{\max}^2 r_{\max},\;
        d_{\max}^2 w_{\max} + d_{\max}^2 r_{\max}\bigr],
\]
where $w_{\min} = \min_i w_{ii}$ and $w_{\max} = \max_i w_{ii}$.
\end{theorem}

\begin{proof}
By Gershgorin's circle theorem, each eigenvalue $\delta$ of $\boldsymbol{\rho}=\boldsymbol{\rho}_t$ lies in at least one disc
\[
  \{z:|z-\rho_{ii}|\le\sum_{j\ne i}|\rho_{ij}|\},
\]
where $\rho_{ii}=w_{ii}d_{i}^2$ and $\rho_{ij}=w_{ij}d_{i}d_{j}$. Thus
\[
  |\delta_i - w_{ii}d_{i}^2|\le d_{i}\sum_{j\ne i}|w_{ij}|\,d_{j}.
\]
Since $\mathbf{W}$ is SPD, $w_{ii}=\sum_k\lambda_k u_{k,i}^2>0$ and each $d_{i}>0$. Hence the real parts of all discs lie strictly to the right of zero, proving $\delta_i>0$.

For the coarse bound, note
\[
  w_{ii}\ge\min_i w_{ii},
  \;
  \sum_{j\ne i}|w_{ij}|\le r_{\max},
  \;
  d_{\min}\le d_{i,t}\le d_{\max}.
\]
So every disc collapses to the real interval (as all eigenvalues are real):
\begin{align*}
D\bigl(w_{ii} d_i^2,\;
&d_i \textstyle\sum_{j \ne i} |w_{ij}|\, d_j\bigr) \\
&\subset
\bigl[
  d_{\min}^2 \min_i w_{ii} - d_{\max}^2 r_{\max}, \\
  &\qquad d_{\max}^2 \max_i w_{ii} + d_{\max}^2 r_{\max}
\bigr].
\end{align*}
Therefore all eigenvalues $\delta_i$ lie in the stated interval.
\end{proof}

\begin{proposition}
\label{prop:hadamard_indefinite}
Let $x_1(t),\dots,x_N(t)$ be distinct real numbers,
\[
J_{ij}(t)=(x_i(t) - x_j(t))^2,\quad J_{ii}(t)=0,
\]
and let $\mathbf{W}\in\mathbb{R}^{N\times N}$ be any symmetric matrix with $w_{ij}\neq0$ for all $i\neq j$. Define the Hadamard product
\[
\boldsymbol{\Omega}(t) = \mathbf{J}(t)\circ \mathbf{W},\qquad \Omega_{ij}(t)=J_{ij}(t)\,w_{ij}.
\]
Then $\boldsymbol{\Omega}(t)$ is symmetric and invertible, yet $\mathrm{tr}(\boldsymbol{\Omega}(t))=0$, so $\boldsymbol{\Omega}(t)$ cannot be positive (semi-)definite.
\end{proposition}

\begin{proof}
Symmetry of $\boldsymbol{\Omega}(t)$ follows immediately from symmetry of $\mathbf{J}(t)$ and $\mathbf{W}$, since
\[
\Omega_{ij}(t)=J_{ij}(t)\,w_{ij}=J_{ji}(t)\,w_{ji}=\Omega_{ji}(t).
\]
Invertibility is guaranteed by the Hadamard rank-lifting argument: because $\mathbf{W}$ has full support, $\det(\mathbf{D}\circ \mathbf{C})\neq0$ for generic such $\mathbf{W}$, hence $\mathrm{rank}(\boldsymbol{\Omega}(t))=N$.

Next, compute the trace:
\[
\mathrm{tr}(\boldsymbol{\Omega}(t))=\sum_{i=1}^N \Omega_{ii}(t)
=\sum_{i=1}^N J_{ii}(t)\,w_{ii}
=\sum_{i=1}^N 0\cdot w_{ii}
=0.
\]
Finally, if $\mathbf{M}$ were positive semidefinite then all its eigenvalues $\{\lambda_k\}$ would satisfy $\lambda_k\ge0$. But their sum is
\[
\sum_{k=1}^N\lambda_k=\mathrm{tr}(\mathbf{M})=0,
\]
forcing each $\lambda_k=0$, contradicting invertibility. Hence $\mathbf{M}$ has both positive and negative eigenvalues and is indefinite.
\end{proof}

\begin{theorem}[Spectral bounds for LDE weighting]
\label{thm:lde_spectral_bounds}
Let $\mathbf{W}\in\mathbb{R}^{N\times N}$ be a real symmetric, full-rank matrix with eigenvalues
\[
\lambda_{\min}(\mathbf{W})\le\cdots\le\lambda_{\max}(\mathbf{W}).
\]
At time $t$, let $\mathbf{x}(t)\in\mathbb{R}^N$ and define the instantaneous squared-difference matrix
\[
J_{ij}(t)=(x_i(t)-x_j(t))^2,\quad J_{ii}(t)=0,
\]
and form the Hadamard-weighted matrix
\[
\boldsymbol{\Omega}(t)=\mathbf{W}\circ \mathbf{J}(t),\quad \Omega_{ij}(t)=w_{ij}\,J_{ij}(t).
\]
Set
\begin{align}
m_t &= \min_{i\neq j}|x_i(t)-x_j(t)|, \\
M_t &= \max_{i\neq j}|x_i(t)-x_j(t)|,
\end{align}
and let $\delta_{1,t}\le\cdots\le\delta_{N,t}$ be the eigenvalues of $\boldsymbol{\Omega}(t)$. Then for each $i=1,\dots,N$,
\[
m_t^2\,\lambda_{\min}(\mathbf{W})\le\delta_{i,t}\le M_t^2\,\lambda_{\max}(\mathbf{W}).
\]
\end{theorem}

\begin{proof}
Let $\mathbf{v}\in\mathbb{R}^N$ be any unit vector, $\|\mathbf{v}\|=1$. The Rayleigh quotient of $\boldsymbol{\Omega}(t)$ at $\mathbf{v}$ is
\[
\mathbf{v}^\top \boldsymbol{\Omega}(t)\,\mathbf{v}
=\sum_{i,j}w_{ij}\,(x_i(t)-x_j(t))^2\,v_i\,v_j.
\]
Since for all $i\neq j$ we have $m_t^2\le(x_i(t)-x_j(t))^2\le M_t^2$, it follows that
\begin{align}
m_t^2\sum_{i,j}w_{ij}v_i v_j
&\le \mathbf{v}^\top \boldsymbol{\Omega}(t)\,\mathbf{v} \notag \\
&\le M_t^2\sum_{i,j}w_{ij}v_i v_j.
\end{align}
But $\sum_{i,j}w_{ij}v_i v_j = \mathbf{v}^\top \mathbf{W}\,\mathbf{v}$, and by the Rayleigh--Ritz theorem
\[
\lambda_{\min}(\mathbf{W})\le \mathbf{v}^\top \mathbf{W}\,\mathbf{v}\le\lambda_{\max}(\mathbf{W}).
\]
Combining these inequalities gives
\[
m_t^2\,\lambda_{\min}(\mathbf{W})\le \mathbf{v}^\top \boldsymbol{\Omega}\,\mathbf{v}
\le M_t^2\,\lambda_{\max}(\mathbf{W}).
\]
Finally, the Courant--Fischer characterization implies that each eigenvalue $\delta_{i,t}$ of $\boldsymbol{\Omega}(t)$ lies within the range of $\mathbf{v}^\top \boldsymbol{\Omega}(t)\,\mathbf{v}$ over unit $\mathbf{v}$. Therefore
\[
m_t^2\,\lambda_{\min}(\mathbf{W})\le\delta_{i,t}\le M_t^2\,\lambda_{\max}(\mathbf{W}),
\]
for $i=1,\dots,N$, as claimed.
\end{proof}

% ====================== THEORETICAL SECTION (drop-in) ======================
% Put these packages in your preamble if not already present:
% \usepackage{amsmath,amssymb,mathtools}
% \usepackage{bm}        % bold math symbols
% \usepackage{amsthm}    % optional, for propositions/lemmas
% \usepackage{microtype} % better line breaking
\subsection{Relation to LoRA and HiRA Adapters}
\label{lorahira}

\begin{figure*}[t]
\centering
\resizebox{\linewidth}{!}{%
\begin{tikzpicture}[line cap=round,line join=round,>=Latex,transform shape]

% ---------------- styles & colors (local) ----------------
\usetikzlibrary{arrows.meta,positioning}
\tikzset{
  blk/.style={rounded corners=3pt, draw=black, very thick,
              minimum width=48mm, minimum height=10mm, align=center},
  smallblk/.style={rounded corners=3pt, draw=black, very thick,
              minimum width=38mm, minimum height=9mm, align=center},
  note/.style={font=\small},
  mapline/.style={densely dashed, very thick, draw=black!55}
}
\definecolor{cgv}{RGB}{198,224,255}   % GVNN blocks
\definecolor{ctf}{RGB}{255,236,179}   % Transformer blocks

% ---------------- column headers ----------------
\node[font=\bfseries] (gtitle) at (0,7.2) {GVNN (data-driven graph tensor)};
\node[font=\bfseries] (ttitle) at (11.4,7.2) {Transformer (data-driven directed graph)};

% ===================== LEFT: GVNN =====================
% Row 1: graph construction from data
\node[blk, fill=cgv] (g1) at (0,5.8)
  {Graph construction from data\\[1pt]
   $\displaystyle \mathbf{J}(t)=F_V\!\big(x_i(t),x_j(t)\big),\quad
   \boldsymbol{\Omega}(t)=\mathbf{W}\circ \mathbf{J}(t)$};

% Row 2: convolution/aggregation
\node[blk, fill=cgv, below=10mm of g1] (g2)
  {Graph convolution (batched)\\[1pt]
   $\displaystyle \mathbf{Z}(t)=\boldsymbol{\Omega}(t)\,\mathbf{X}(t)$};

% Row 3: mixing + nonlinearity
\node[smallblk, fill=cgv, below=10mm of g2] (g3)
  {Mixing + nonlinearity\\[1pt]
   $\displaystyle \mathbf{Y}(t)=\sigma\!\big(\boldsymbol{\Theta}[\,a_t \mathbf{X}(t)+b_t \mathbf{Z}(t)\,]\big)$};

% Vertical flow
\draw[->, very thick] (g1.south) -- (g2.north);
\draw[->, very thick] (g2.south) -- (g3.north);

% ==================== RIGHT: TRANSFORMER ==================
% Row 1: attention graph construction (directed)
\node[blk, fill=ctf, anchor=west] (t1) at ([xshift=65mm] g1.east)
  {Attention graph from data\\[1pt]
   $\displaystyle \mathbf{Q}=\mathbf{X}\mathbf{W}_Q,\;\mathbf{K}=\mathbf{X}\mathbf{W}_K,\;\mathbf{V}=\mathbf{X}\mathbf{W}_V,\quad
   \mathbf{A}=\mathrm{softmax}\!\left(\frac{\mathbf{Q}\mathbf{K}^\top}{\sqrt{d_k}}\right)$};

% Row 2: attention aggregation (convolution view)
\node[blk, fill=ctf, below=10mm of t1] (t2)
  {Aggregation as graph convolution\\[1pt]
   $\displaystyle \mathbf{A}\cdot \mathbf{V}$};

% Row 3: residual + nonlinearity
\node[smallblk, fill=ctf, below=10mm of t2] (t3)
  {Residual + nonlinearity\\[1pt] Add \& Norm + FFN};

% Vertical flow
\draw[->, very thick] (t1.south) -- (t2.north);
\draw[->, very thick] (t2.south) -- (t3.north);

% =================== CORRESPONDENCE (simple dashed lines) =========
\draw[mapline] (g1.east) -- node[midway, above, note]{data-driven graph} (t1.west);
\draw[mapline] (g2.east) -- node[midway, above, note]{graph convolution} (t2.west);
\draw[mapline] (g3.east) -- node[midway, above, note]{nonlinear mixing / head} (t3.west);

\end{tikzpicture}
}% end resizebox
\caption{Both architectures construct a graph from the input and then convolve over it.
GVNN forms a data-driven adjacency tensor $\boldsymbol{\Omega}(t)=\mathbf{W}\!\circ\!\mathbf{J}(t)$ and performs
$\mathbf{Z}(t)=\boldsymbol{\Omega}(t)\mathbf{X}(t)$ before a learned mixing and nonlinearity
$\mathbf{Y}(t)=\sigma\!\big(\boldsymbol{\Theta}[\,a_t \mathbf{X}(t)+b_t \mathbf{Z}(t)\,]\big)$.
A Transformer builds a directed, data-driven attention graph
$\mathbf{A}=\mathrm{softmax}(\mathbf{Q}\mathbf{K}^\top/\sqrt{d_k})$ and aggregates via $\mathbf{A}\cdot \mathbf{V}$,
followed by residual connections and a feed-forward network.}
\label{fig:data-driven-graphs-clean}
\end{figure*}

Parameter-efficient fine-tuning (PEFT) adapts large models by training only a small number of parameters. 
\textbf{LoRA} \citep{hu2021lora} achieves this by expressing the update as a low-rank factorization, $\Delta \mathbf{W} = \mathbf{A}\mathbf{B}$ with $\mathrm{rank}(\Delta \mathbf{W})\le r$, trading full expressiveness for efficiency. 
\textbf{HiRA} \citep{huang2025hira} increases expressiveness without sacrificing PEFT by applying a Hadamard (elementwise) product between a high-rank base and a low-rank factor: 
\begin{equation}
\Delta \mathbf{W} = \mathbf{W}_0 \odot (\mathbf{A}\mathbf{B}),
\end{equation}
with $\mathrm{rank}(\Delta \mathbf{W})\le \mathrm{rank}(\mathbf{W}_0)\cdot\mathrm{rank}(\mathbf{A}\mathbf{B})$.
This allows the update to attain a much higher effective rank while keeping trainable parameters comparable to LoRA.

\textbf{GVNNs} leverage the same algebraic idea. At each time step, an instantaneous (often low-rank) connectivity $\mathbf{J}_t$ is fused with a stable, typically high-rank support $\mathbf{W}$ via a Hadamard product:
\[
\boldsymbol{\Omega}_t = \mathbf{W} \odot \mathbf{J}_t.
\]
This multiplicative fusion boosts the rank and stabilizes $\boldsymbol{\Omega}_t$, ensuring a more expressive operator even when $\mathbf{J}_t$ is rank-deficient. 

In fact, the support $\mathbf{W}$ need not be fixed. In analogy to LoRA, one can parameterize $\mathbf{W}$ itself as 
\[
\mathbf{W} = \mathbf{W}_{\mathrm{base}} + \Delta \mathbf{W}, 
\quad
\Delta \mathbf{W} = \mathbf{A} \mathbf{B},
\]
where $\mathbf{W}_{\mathrm{base}}$ is an initialization (e.g., long-term correlation) and $\Delta \mathbf{W}$ is a low-rank adapter. 
This formulation enables \emph{efficient adaptation of the support} while avoiding the cost of learning a full $N \times N$ matrix. 

Alternatively, in a HiRA-style design, we may define
\[
\mathbf{W} = \mathbf{W}_{\mathrm{base}} \odot (\mathbf{A}\mathbf{B}),
\]
so that the expressive capacity of the Hadamard product is preserved even when $\mathbf{A}\mathbf{B}$ is low-rank. 

This perspective shows that the Hadamard support in GVNNs can itself be learned using LoRA/HiRA adapters: low-rank updates capture task-specific variations, while the Hadamard structure ensures that these updates interact multiplicatively with instantaneous connectivities $\mathbf{J}_t$. In practice, this allows GVNNs to scale to large graphs without incurring prohibitive parameter costs, while retaining the flexibility to adapt supports across datasets and tasks.

\subsection{Transformers are Graph Variate Neural Networks (and Vice Versa)}
\label{sec:theory-data-driven-graphs}

Recent work has suggested that the transformer model is in fact a graph neural network that has \textit{`won the hardware lottery'}. This suggests that we can, in fact, go the other direction and build better Graph Neural Network architectures by leveraging ideas from the transformer model.

The following discussion will demonstrate that the transformer architecture is in fact not only a Graph Neural Network but in fact a \textit{Graph Variate Neural Network}, i.e.\ one whose core operation is an input dependent graph convolution. In fact, the transformer block can be reinterpreted as a GVNN with a static graph variable tensor (i.e.\ the attention matrix replicated over all $T$) just with differences in normalization and linear weight projections.

\subsubsection{Transformer Self-Attention as Directed Data-Driven Graph Convolution}

Given token features $\mathbf{X} \in \mathbb{R}^{T \times d}$, the Transformer computes
\emph{queries}, \emph{keys}, and \emph{values}
\begin{equation}
\mathbf{Q} = \mathbf{X} \mathbf{W}_Q,\qquad \mathbf{K} = \mathbf{X} \mathbf{W}_K,\qquad \mathbf{V} = \mathbf{X} \mathbf{W}_V,
\end{equation}
then forms a \emph{row-stochastic, directed} attention matrix
\begin{equation}
\label{eq:attn}
\mathbf{A} \;=\; \mathrm{softmax}\!\left(\frac{\mathbf{Q}\mathbf{K}^\top}{\sqrt{d_k}}\right)
\;\in\; \mathbb{R}^{T\times T}\!,
\end{equation}
and aggregates values via
\begin{equation}
\label{eq:attn-agg}
\mathrm{Attn}(\mathbf{X}) \;=\; \mathbf{A}\,\mathbf{V} \;\in\; \mathbb{R}^{T\times d_v}.
\end{equation}
Equations~\eqref{eq:attn}--\eqref{eq:attn-agg} implement
\emph{graph convolution on a data-driven, directed graph} whose adjacency is $\mathbf{A}$:
each row of $\mathbf{A}$ defines outgoing edges from a token to all others with weights
given by the softmax of similarities. Residual connections and a position-wise
feed-forward network complete the encoder block.

\paragraph{Multi-head attention.}
For $H$ heads with $\mathbf{A}^{(h)}$ and $\mathbf{V}^{(h)}$, the aggregation is
$\mathrm{Concat}_h\big(\mathbf{A}^{(h)}\mathbf{V}^{(h)}\big) \mathbf{W}_O$, a parallel sum of
graph convolutions on $H$ distinct data-driven graphs.

\subsubsection{GVNN as Data-Driven Graph Convolution}
GVNN constructs a \emph{graph-variate tensor} via two ingredients:

\begin{enumerate}
\item A \textbf{node-wise similarity/interaction} functional
$F_V:\mathbb{R}\times\mathbb{R}\to\mathbb{R}$ producing
\begin{equation}
\label{eq:J}
J_{ij}(t) \;=\; F_V\!\big(x_i(t),x_j(t)\big) \quad\Rightarrow\quad
\mathbf{J}(t)\in\mathbb{R}^{N\times N}.
\end{equation}
Examples include the LDE and instantaneous correlation.

\item A \textbf{stable support} $\mathbf{W} \in \mathbb{R}^{N \times N}$ (fixed or learned) that
encodes long-term topology or sparsity. GVNN forms the pointwise (Hadamard)
product
\begin{equation}
\label{eq:omega}
\boldsymbol{\Omega}(t) \;=\; \mathbf{W} \circ \mathbf{J}(t)\,,
\end{equation}
which gates/filters instantaneous interactions by the support.
\end{enumerate}

Given $\boldsymbol{\Omega}(t)$, GVNN performs a batched graph convolution of the current
signal:
\begin{equation}
\label{eq:gvnn-conv}
\mathbf{z}(t) \;=\; \boldsymbol{\Omega}(t)\,\mathbf{x}(t) \;\in\; \mathbb{R}^{N}.
\end{equation}
A compact GVNN layer then mixes the original and aggregated signals followed by
a nonlinearity:
\begin{equation}
\label{eq:gvnn-layer}
\mathbf{y}(t) \;=\; \sigma\!\Big(\boldsymbol{\Theta}\,[\,a_t\,\mathbf{x}(t) \;+\; b_t\,\mathbf{z}(t)\,]\Big),
\end{equation}
where $\boldsymbol{\Theta} \in \mathbb{R}^{N\times N}$ is a learned linear map (or small MLP),
and $a_t,b_t$ are (optionally learned) scalar/broadcast coefficients. Stacking
$L$ layers yields $\mathbf{h}^{(l)}(t)$ with $\mathbf{h}^{(0)}(t)=\mathbf{x}(t)$ and
\[
\boldsymbol{\Omega}^{(l)}(t)=\mathbf{W}\circ \mathbf{J}^{(l)}(t),\qquad
J^{(l)}_{ij}(t)=F_V\!\big(h_i^{(l-1)}(t),h_j^{(l-1)}(t)\big).
\]

\paragraph{Multi-node function convolution.}
Similar to multi-head attention, one may aggregate convolutions with different node functions and stable supports.

\ifCLASSOPTIONcaptionsoff
  \newpage
\fi

\end{document}